\documentclass[11pt,letterpaper]{article}
\usepackage[margin=1in]{geometry}
\usepackage[T1]{fontenc}
\usepackage{lmodern}
\usepackage{amsmath,amssymb,amsthm}
\usepackage{graphicx,booktabs,array,microtype}
\usepackage{enumitem,placeins,longtable,tikz,xcolor}
\usepackage[round,authoryear]{natbib}
\usepackage{hyperref}
\usetikzlibrary{arrows.meta,positioning}
\hypersetup{hidelinks,pdftitle={Finite-Sample Optimal Conformal Prediction Sets under Covariate Shift},pdfauthor={Zhongyang Li}}
\bibpunct{(}{)}{;}{a}{,}{,}
\newtheorem{theorem}{Theorem}
\newtheorem{lemma}[theorem]{Lemma}
\newtheorem{proposition}[theorem]{Proposition}
\newtheorem{corollary}[theorem]{Corollary}
\newtheorem{remark}[theorem]{Remark}
\newcommand{\E}{\mathbb E}
\newcommand{\Prob}{\mathbb P}
\newcommand{\1}{\mathbf 1}

\newcommand{\clip}[1]{\left[#1\right]_0^1}

\allowdisplaybreaks[1]

\setlist{topsep=4pt,itemsep=2pt,parsep=0pt}
\title{Finite-Sample Optimal Conformal Prediction Sets\\
under Covariate Shift}
\author{Zhongyang Li\\[3pt]
\small Department of Mathematics, University of Connecticut\\
\small 341 Mansfield Road, Unit 1009, Storrs, CT 06269-1009, USA\\
\small \href{mailto:lizhongyangmath@gmail.com}{\texttt{lizhongyangmath@gmail.com}}\\
\small \url{https://mathzhongyangli.wordpress.com/}}
\date{}
\begin{document}
\maketitle

\begin{abstract}
We study conformal prediction under covariate shift with a known
covariate density ratio, seeking prediction sets with minimum
expected length under the target distribution.
For every calibration sample size, we prove that the rule combining
the inverse conditional-density score with randomized weighted
conformal calibration minimizes this expected length among all
rules valid over the full weighted-exchangeability class.

We approximate the oracle using quantile and kernel histograms,
including forest-weighted estimators, and select estimators and
tuning parameters by resampling to estimate calibrated target length.
On a bounded response domain, we bound excess length linearly in
$L^1$ density error at fixed calibration size, without requiring
a uniform positive lower bound on the true density.
We also bound selection-pool and resampling errors.
Under explicit growth conditions, the selected procedure approaches
the same oracle while the calibration size remains fixed.

Experiments on 36 known-density configurations and four public
regression datasets compare coverage, prediction-set length, and
runtime with existing methods.
{
In targeted low-coverage and strong-shift comparisons, inverse-density
selection yields shorter mean sets in most cases and lower mean runtime
than the evaluated weighted baselines, with some differences in achieved
coverage.
}
\end{abstract}
\noindent\textbf{Keywords:}
conformal prediction, covariate shift, prediction sets,
finite-sample optimality, model selection
\par

\section{Introduction}\label{sec:intro}

Conformal prediction provides finite-sample marginal coverage under
exchangeability without a correctly specified prediction model
\citep{shafer2008tutorial,angelopoulos2023gentle}. Under covariate shift,
calibration and test laws $P,Q$ share the conditional response density
$f(y\mid x)$ but have different input distributions. Weighted calibration
retains target coverage when $w=dQ_X/dP_X$ is known
\citep{tibshirani2019shift}. We study which calibrated rule minimizes
target expected prediction-set length at a fixed calibration size.

Existing methods use quantile regression \citep{romano2019cqr},
conditional distributions or densities
\citep{chernozhukov2021dcp,izbicki2020density,izbicki2022hpd,wang2025density},
quantile histograms \citep{sesia2021histograms,luo2025cti},
length-oriented optimization \citep{kiyani2024length,lebars2025volume},
or unions of a prescribed number of intervals \citep{gao2025volume}.
Their coverage constraints and optimality benchmarks differ.
For finite-sample optimality, we build on
\citet*[Theorem~7]{tibshirani2026testing} and its precedents
\citep{hoff2023,koning2025fuzzy}.

Our main contributions are:
\begin{enumerate}[leftmargin=*]
\item \emph{Weighted finite-sample optimality.}
For every calibration size $m$, randomized weighted calibration with
score $1/f(y\mid x)$ minimizes target expected volume among all rules
valid over the full class of weighted permutation mixtures
(Theorem~\ref{wo:oracle}). Risk is evaluated under $P^m\times Q$;
validity is required over the full weighted class. The density ratio
enters calibration but cancels from the optimal score ordering.
Prediction sets may be disconnected.

\item \emph{Estimation and excess length.}
We bound excess length by score error and, on a finite-volume response
domain, by density error without a uniform positive density lower bound
or clipping (Section~\ref{sec:approximation}). For normalized estimates,
Proposition~\ref{prop:fixed-m-density} gives a linear $L^1$ bound at fixed
$m$ and an example showing that its error order cannot be improved
uniformly. Quantile and kernel histogram bounds, including forest weights
(Section~\ref{sec:algorithm}), yield fixed-$m$ consistency under the
conditions in Section~\ref{sec:forest-fixed-m-consistency}.

\item \emph{Selection by calibrated target length.}
Building on conformal predictor selection
\citep{yang2025selection,liang2026selection}, we analyze bootstrap
calibration from a labeled pool and length evaluation on an independent
input pool (Section~\ref{sec:pool-selection}). Separating approximation,
pool, and resampling errors gives consistency after selection toward the
same fixed-$m$ oracle under explicit growth conditions
(Corollary~\ref{cor:forest-selected-fixed-m}). Independent final
calibration retains exact target marginal coverage.
\end{enumerate}

{
Figure~\ref{fig:length-selection-workflow} summarizes inverse-density
prediction with length selection or negative log-likelihood (NLL)
selection.

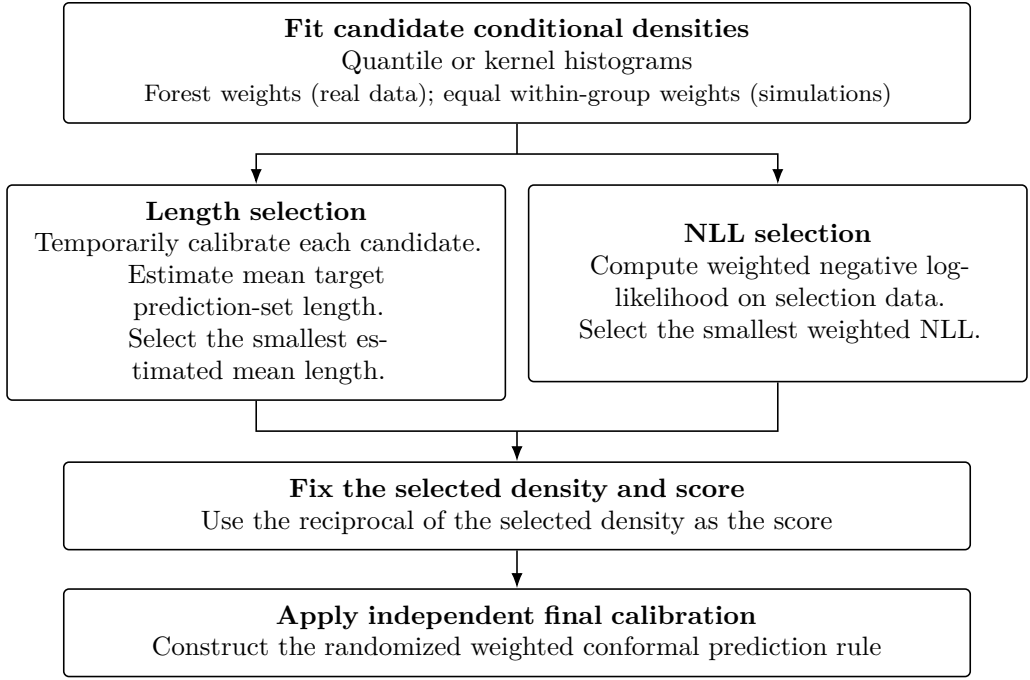
\begin{figure}[!htbp]
\centering
\begin{tikzpicture}[
  node distance=5mm,
  workflow box/.style={
    draw=black, text=black, rounded corners=2pt,
    line width=0.6pt, align=center,
    text width=11.5cm, inner xsep=7pt, inner ysep=6pt,
    font=\small
  },
  selection box/.style={
    workflow box, text width=6.1cm, minimum height=26mm
  },
  workflow arrow/.style={
    -{Latex[length=2mm]}, draw=black, line width=0.6pt
  },
  workflow line/.style={draw=black, line width=0.6pt}
]
\node[workflow box] (fit) {
  \textbf{Fit candidate conditional densities}\\
  Quantile or kernel histograms\\
  {\footnotesize Forest weights (real data);
  equal within-group weights (simulations)}
};
\node[selection box, below=8mm of fit, xshift=-3.45cm] (length) {
  \textbf{Length selection}\\
  Temporarily calibrate each candidate.\\
  Estimate mean target prediction-set length.\\
  Select the smallest estimated mean length.
};
\node[selection box, below=8mm of fit, xshift=3.45cm] (nll) {
  \textbf{NLL selection}\\
  Compute weighted negative log-likelihood
  on selection data.\\
  Select the smallest weighted NLL.
};
\node[workflow box, below=8mm of length, xshift=3.45cm] (fix) {
  \textbf{Fix the selected density and score}\\
  Use the reciprocal of the selected density as the score
};
\node[workflow box, below=of fix] (calibrate) {
  \textbf{Apply independent final calibration}\\
  Construct the randomized weighted conformal prediction rule
};
\coordinate (split) at ([yshift=-4mm]fit.south);
\coordinate (join) at ([yshift=4mm]fix.north);
\draw[workflow line] (fit.south) -- (split);
\draw[workflow arrow] (split) -| (length.north);
\draw[workflow arrow] (split) -| (nll.north);
\draw[workflow line] (length.south) |- (join);
\draw[workflow line] (nll.south) |- (join);
\draw[workflow arrow] (join) -- (fix.north);
\draw[workflow arrow] (fix) -- (calibrate);
\end{tikzpicture}
\caption{{Alternative selection criteria for inverse-density
prediction. Length selection is the primary method; NLL selection is
the comparator. Each branch selects one candidate and proceeds
separately to final calibration. Training, selection, and final
calibration use independent data.
See Sections~\ref{sec:algorithm} and~\ref{sec:selection}.}}
\label{fig:length-selection-workflow}
\end{figure}
}

Simulations compare fitted rules with the finite-$m$ oracle in 36
known-density configurations. A further study removes the density floor,
fixes $m=200$, and increases training, selection pools, and bootstrap
repetitions. Four public datasets compare coverage, transformed-scale
length, and runtime; controlled comparisons separate score and
density-selection effects. Highest predictive density (HPD) scores
\citep{izbicki2022hpd} measure the estimated conditional probability mass
of responses with higher density than the candidate response.
In targeted comparisons, inverse-density selection yields shorter mean
sets than our weighted HPD implementation in most tested low-coverage
settings and than our weighted implementation of distributional conformal
prediction optimized for short intervals (DCP-opt)
\citep{chernozhukov2021dcp} in most tested strong-shift settings.
Measured mean runtime is lower throughout these comparisons
(Section~\ref{sec:dcp-shift10}; Appendix~\ref{app:hpd-low-coverage}).
These comparisons use common nominal coverage levels, with some
differences in achieved coverage.

Sections~\ref{sec:setting}--\ref{sec:selection} develop the framework and
theory; experiments and related work follow in
Sections~\ref{sec:experiments} and~\ref{sec:related}.

\section{Setting and weighted calibration}\label{sec:setting}

Let $Z=(X,Y)$ take values in a standard Borel space
$\mathcal X\times\mathcal Y$. The calibration sample
$T=(Z_1,\ldots,Z_m)$ is iid from $P$. Independently, the test pair has law
$Q$, where
\begin{equation}\label{eq:shift}
 dQ(x,y)=w(x)\,dP(x,y),\qquad 0\le w(x)<\infty,\qquad \E_Pw(X)=1.
\end{equation}
This is covariate shift: $Q_X$ can differ from $P_X$, while
$Q_{Y\mid X}=P_{Y\mid X}$. The known weight is the Radon--Nikodym
derivative $dQ_X/dP_X$. 

Write $N=m+1$ and $p=1-\alpha$, where $0<\alpha<1$. Let \(\mu\) be a \(\sigma\)-finite measure on \(\mathcal Y\). For \(Y\in\mathbb R\), we take \(\mu\) to be Lebesgue measure. Write \(\mathcal Z=\mathcal X\times\mathcal Y\), and let \(\mathcal B(\mathcal Y)\) denote the collection of Borel subsets of \(\mathcal Y\).

To allow randomization, we represent a prediction rule by a fixed set-valued map
\[
C:\mathcal Z^m\times\mathcal X\times[0,1]
\longrightarrow\mathcal B(\mathcal Y),
\]
whose membership indicator
\[
(t,x,u,y)\longmapsto\mathbf 1\{y\in C(t,x,u)\}
\]
is jointly measurable. Given a calibration sample \(T\) and a test input \(x\), the rule returns \(C(T,x,U)\), where \(U\sim\operatorname{Unif}[0,1]\) is independent of all data. Nonrandomized rules are included as the special case in which \(C(t,x,u)\) does not depend on \(u\).

For fixed $t,x,y$, we write
\[
\begin{aligned}
\mathbb P_U\{y\in C(t,x,U)\}
&:=\int_0^1 \mathbf 1\{y\in C(t,x,u)\}\,du,\\
\mathbb E_U[\mu(C(t,x,U))]
&:=\int_0^1 \mu(C(t,x,u))\,du.
\end{aligned}
\]
Thus the subscript $U$ indicates that only the auxiliary
random variable $U$ is integrated out; the data remain fixed.

The target risk of \(C\) is
\begin{align}
R_m(C)
=
\mathbb E_{T\sim P^m,\,X\sim Q_X,\,U}
\bigl[\mu(C(T,X,U))\bigr].\label{eq:risk}
\end{align} 
We condition on all {training}
and selection data until those steps are analyzed explicitly.

Fix a measurable finite score $s(x,y)$, with smaller values preferred.
Set $s_i=s(X_i,Y_i)$ and $w_i=w(X_i)$. The target inputs have $w(X)>0$
$Q_X$-almost surely. At a zero-weight input, define the output as
$\mathcal Y$; such inputs do not affect target risk or coverage.
For a candidate $(x,y)$ with $w(x)>0$, write
$W_x=\sum_{i=1}^m w_i+w(x)$ and
\begin{align}
 A_s(T,x,y)&=\sum_{i=1}^m w_i\1\{s_i>s(x,y)\},\label{eq:A}\\
 B_s(T,x,y)&=w(x)+\sum_{i=1}^m w_i\1\{s_i=s(x,y)\}.\label{eq:B}
\end{align}
The weighted randomized conformal rule is
\begin{equation}\label{eq:calibrated}
 C_s(T,x,U)=
 \left\{y:\frac{A_s(T,x,y)+U B_s(T,x,y)}{W_x}>\alpha\right\},
 \qquad U\sim\operatorname{Unif}[0,1].
\end{equation}
To construct the prediction set at a test input \(x\), draw \(U\sim\operatorname{Unif}[0,1]\) independently of the data and hold its value fixed as \(y\) varies over \(\mathcal Y\).

Its inclusion probability has the closed form
\begin{equation}\label{eq:inclusion}
 r_s(T,x;y)=\Prob_U\{y\in C_s(T,x,U)\}
 =\clip{\frac{pW_x-\sum_iw_i\1\{s_i<s(x,y)\}}
                   {w(x)+\sum_iw_i\1\{s_i=s(x,y)\}}},
\end{equation}
where $[a]_0^1=\min\{1,\max\{0,a\}\}$. Consequently,
\begin{equation}\label{eq:integrated-length}
 \E_U\mu(C_s(T,x,U))=\int_{\mathcal Y}r_s(T,x;y)\,\mu(dy).
\end{equation}

The guarantee below is marginal over calibration data, the target
observation, and $U$. It does not require $p$ coverage at every input
$x$, or conditional on one realized calibration sample.

\section{A weighted version of the finite-sample oracle theorem}\label{wo:section}

\subsection{Validity class and oracle theorem}

For an augmented sample $D=(z_1,\ldots,z_N)\in\mathcal Z^N$, with $z_i=(x_i,y_i)$,
let $w_i=w(x_i)$ and
\[
W(D)=\sum_{i=1}^N w_i.
\]
For each positive integer $k$, let $\mathcal S_k$ denote
the set of permutations of $\{1,\ldots,k\}$.
For $\sigma\in\mathcal S_N$, write
\[
D_\sigma=(z_{\sigma(1)},\ldots,z_{\sigma(N)}).
\] 
For $W(D)>0$ and any Borel set $A\subseteq\mathcal Z^N$, define
\begin{equation}\label{wo:orbit}
K_w(D,A)
=
\frac{1}{m!\,W(D)}
\sum_{\sigma\in\mathcal S_N}
w(x_{\sigma(N)})
\mathbf 1\{D_\sigma\in A\}.
\end{equation}
For each $D$ with $W(D)>0$, $K_w(D,\cdot)$ is a probability
measure on $\mathcal Z^N$.
Let $\mathcal H_w$ be the class of all mixtures of these measures:
\begin{equation}\label{wo:null-class}
\mathcal H_w
=
\left\{
\int_{\{W>0\}} K_w(D,\cdot)\,\nu(dD)
:
\nu\{W>0\}=1
\right\},
\end{equation}
where $\nu$ ranges over probability measures on $\mathcal Z^N$.

For a joint distribution $H$ on $\mathcal Z^N$, we say
that $C$ has coverage at least $p$ under $H$ if
\[
\mathbb P_{H,U}\{Y_N\in C(T,X_N,U)\}\ge p,
\]
where $(Z_1,\ldots,Z_N)\sim H$, $Z_i=(X_i,Y_i)$,
$T=(Z_1,\ldots,Z_m)$,
and $U\sim\operatorname{Unif}[0,1]$ is independent of
the sample. The probability is taken over both the
sample and $U$.

The iid covariate-shift model in Section~\ref{sec:setting} satisfies
$P^m\times Q\in\mathcal H_w$, as shown in
Lemma~\ref{lem:covariate-shift-inclusion} below.

Assume that the reference conditional law has a jointly measurable
density $f(y\mid x)=dP_{Y\mid X=x}/d\mu(y)$ that is positive and finite
$(P_X\otimes\mu)$-almost everywhere. On that set define
\[
 s^*(x,y)=1/f(y\mid x).
\]
Extend $s^*$ as any finite positive measurable function on the exceptional
set, for example by setting it equal to one. The score is then defined
on every augmented sample. Write $C^*=C_{s^*}$, with $C_s$ defined
in \eqref{eq:calibrated}. We call $C^*$ the oracle rule
because its score uses the true conditional density $f$.
Its target risk $R_m(C^*)$, defined in \eqref{eq:risk},
does not depend on the choice of extension.

\begin{theorem}[Weighted finite-sample oracle]\label{wo:oracle}
Among all measurable randomized prediction rules with coverage at least
$p$ under every $H\in\mathcal H_w$, the minimum target risk
\eqref{eq:risk} is attained by $C^*$:
\begin{equation}\label{wo:oracle-optimality}
 R_m(C^*)=
 \inf_{C:\,\inf_{H\in\mathcal H_w}\mathbb P_{H,U}
       \{Y_N\in C(T,X_N,U)\}\ge p} R_m(C).
\end{equation}
The result holds for every finite $m$ and every
$\sigma$-finite $\mu$, and the risks may equal $+\infty$. The oracle rule $C^*$ has exact coverage $p$
over $\mathcal H_w$.
\end{theorem}

\subsection{Proof of the oracle theorem}

We first establish three facts used in the proof.

\begin{lemma}[Inclusion of the iid covariate-shift model]
\label{lem:covariate-shift-inclusion}
Define a measure $\nu_P$ on $\mathcal Z^N$ by
\[
\nu_P(dD)=\frac{W(D)}{N}\,P^N(dD).
\]
Then $\nu_P$ is a probability measure satisfying
$\nu_P\{W>0\}=1$, and
\begin{equation}\label{wo:product-mixture}
P^m\times Q
=
\int_{\{W>0\}} K_w(D,\cdot)\,\nu_P(dD).
\end{equation}
In particular, $P^m\times Q\in\mathcal H_w$.
\end{lemma}

\begin{proof}
Since $\mathbb E_Pw(X)=1$, we have
$\nu_P(\mathcal Z^N)=\mathbb E_{P^N}W/N=1$ and
$\nu_P\{W=0\}=0$. For any Borel set $A$, substitution gives
\[
\begin{aligned}
 \int_{\{W>0\}}K_w(D,A)\,\nu_P(dD)
 &=\frac1{N!}\sum_{\sigma\in\mathcal S_N}
   \int w(x_{\sigma(N)})\mathbf1\{D_\sigma\in A\}\,P^N(dD)\\
 &=\int_A w(x_N)\,P^N(dD)
 =(P^m\times Q)(A).
\end{aligned}
\]
The integrands vanish when $W=0$, justifying integration over all
$D$. Permutation invariance of $P^N$ makes all $N!$ terms equal,
and the last equality uses $dQ=w\,dP$.
\end{proof}

\begin{proposition}[Coverage over $\mathcal H_w$]
\label{wo:equivalence}
Let $C$ be a prediction rule as defined in
Section~\ref{sec:setting}
whose inclusion probabilities do not depend on the order of
the calibration sample:
\begin{equation}\label{wo:order-symmetry}
\mathbb P_U\{y\in C(t_\pi,x,U)\}
=
\mathbb P_U\{y\in C(t,x,U)\}
\end{equation}
for all $t\in\mathcal Z^m$, $x\in\mathcal X$,
$y\in\mathcal Y$, and $\pi\in\mathcal S_m$, where
$t_\pi=(z_{\pi(1)},\ldots,z_{\pi(m)})$ for
$t=(z_1,\ldots,z_m)$.

For $D=(z_1,\ldots,z_N)$, let $D_{-i}$ denote the sample
obtained by deleting $z_i=(x_i,y_i)$, and define
\[
e_i^C(D)=\mathbb P_U\{y_i\in C(D_{-i},x_i,U)\}.
\]
Then $C$ has coverage at least $p$ under every
$H\in\mathcal H_w$ if and only if
\begin{equation}\label{wo:bag-constraint}
\sum_{i=1}^N w_i e_i^C(D)\ge pW(D)
\qquad\text{for every }D\in\mathcal Z^N.
\end{equation}

For any prediction rule $C_0$, define
\[
 C_0^\circ(t,x;U,\Pi)=C_0(t_\Pi,x,U),
 \qquad \Pi\sim\operatorname{Unif}(\mathcal S_m),
\]
with $\Pi$ independent of the data and $U$. Its inclusion probability is
\[
 \frac1{m!}\sum_{\pi\in\mathcal S_m}
 \mathbb P_U\{y\in C_0(t_\pi,x,U)\},
\]
so it satisfies \eqref{wo:order-symmetry}. It has the same coverage as
$C_0$ under every $H\in\mathcal H_w$ and the same target risk
$R_m(C_0^\circ)=R_m(C_0)$, with its randomization averaged over both
$U$ and $\Pi$.
\end{proposition}

\begin{proof}
For $W(D)>0$, exactly $m!$ permutations in \eqref{wo:orbit}
put observation $i$ last. Calibration-order invariance therefore gives
\begin{equation}\label{wo:kernel-coverage}
 \mathbb P_{K_w(D,\cdot),U}
 \{Y_N\in C(T,X_N,U)\}
 =\frac{\sum_iw_i e_i^C(D)}{W(D)}.
\end{equation}
Each $K_w(D,\cdot)$ belongs to $\mathcal H_w$, and every member of
$\mathcal H_w$ is a mixture of these kernels. Thus coverage over the
class is equivalent to \eqref{wo:bag-constraint}. When $W(D)=0$,
all weights vanish and the constraint holds automatically.

Reordering $t$ only reindexes the average defining $C_0^\circ$.
To check that symmetrization preserves coverage, let
$\rho\in\mathcal S_N$ fix $N$. The change of index
$\sigma\mapsto\sigma\circ\rho$ in \eqref{wo:orbit} leaves both
$w(x_{\sigma(N)})$ and the kernel unchanged under this reordering.
Hence each kernel, and every mixture $H$, is invariant under
permutations of the first $m$ coordinates. Therefore
$(T_\pi,X_N,Y_N)$ and $(T,X_N,Y_N)$ have the same law under $H$,
which proves equality of coverage after averaging over $\Pi$.
Under $P^m\times Q_X$, $(T_\pi,X)$ and $(T,X)$ likewise have the
same law. Averaging $\mu(C_0(T_\pi,X,U))$ over $\Pi$ proves
$R_m(C_0^\circ)=R_m(C_0)$, including infinite risks.
\end{proof}

\begin{proposition}[Exact randomized weighted calibration]
\label{wo:exact}
Let $s:\mathcal X\times\mathcal Y\to\mathbb R$ be a fixed
measurable score function.
Let $C_s$ be defined by \eqref{eq:calibrated} when
$W_x>0$, and by $C_s(t,x,u)=\mathcal Y$ when $w_X=0$,
as specified in Section~\ref{sec:setting}.
Then $C_s$ satisfies
\eqref{wo:bag-constraint} with equality. Its coverage is
exactly $p$ under every $H\in\mathcal H_w$, and in particular
under $P^m\times Q$.
\end{proposition}
\begin{proof}
Fix $D=(z_1,\ldots,z_N)$ with $W=W(D)>0$, and write
$s_i=s(x_i,y_i)$.
List the distinct scores of the positive-weight observations
in decreasing order:
\[
\{a_1,\ldots,a_K\}
=
\{s_i:1\le i\le N,\ w_i>0\},
\qquad a_1>\cdots>a_K.
\]
Define
\[
b_k=\frac{1}{W}\sum_{i:s_i=a_k}w_i,
\qquad
q_k=\sum_{\ell<k}b_\ell.
\]
Here $b_k$ is the fraction of total weight at score $a_k$,
and $q_k$ is the fraction at larger scores.
Since $b_k>0$ and $\sum_k b_k=1$, the intervals
$[q_k,q_k+b_k]$ cover $[0,1]$ without overlapping interiors.

For $w_i>0$ and $s_i=a_k$, applying
\eqref{eq:calibrated} with calibration sample $D_{-i}$
and candidate $(x_i,y_i)$ gives
\[
e_i^{C_s}(D)
=
\mathbb P_U\{q_k+Ub_k>\alpha\}.
\]
Grouping the weighted sum by score, we obtain
\[
\begin{aligned}
\frac{1}{W}\sum_{i=1}^N w_i e_i^{C_s}(D)
&=
\sum_{k=1}^K b_k
\int_0^1 \mathbf 1\{q_k+ub_k>\alpha\}\,du\\
&=
\sum_{k=1}^K
\int_{q_k}^{q_k+b_k}\mathbf 1\{v>\alpha\}\,dv\\
&=
\int_0^1\mathbf 1\{v>\alpha\}\,dv
=
1-\alpha=p.
\end{aligned}
\]
Indices with $w_i=0$ contribute nothing.
If $W(D)=0$, then all $w_i=0$. By the zero-weight
convention in Section~\ref{sec:setting},
$e_i^{C_s}(D)=1$ for every $i$, so both sides of
\eqref{wo:bag-constraint} are zero.

Reordering the calibration sample does not change $C_s$.
Hence \eqref{wo:kernel-coverage} gives coverage exactly
$p$ under each $K_w(D,\cdot)$ and therefore under every
mixture $H\in\mathcal H_w$.
The conclusion for $P^m\times Q$ follows from
Lemma~\ref{lem:covariate-shift-inclusion}.
\end{proof}

\begin{proof}[Proof of Theorem~\ref{wo:oracle}]
By Proposition~\ref{wo:equivalence}, it suffices to consider
a rule $C$ satisfying \eqref{wo:order-symmetry} and
\eqref{wo:bag-constraint}.
Write
$e_i=e_i^C(D)$, $e_i^*=e_i^{C^*}(D)$, and
$s_i^*=s^*(X_i,Y_i)$.
Tonelli's theorem and permutation symmetry under $P^N$ give
\begin{equation}\label{wo:risk-identity}
 R_m(C)
 =\mathbb E_{P^N}[w_Ns_N^*e_N]
 =\frac1N\mathbb E_{P^N}\sum_{i=1}^Nw_is_i^*e_i.
\end{equation}
The first identity follows because $s^*=1/f$ cancels
the conditional response density. The same identity
holds for $C^*$.

Fix an augmented sample $D$, and write $W=\sum_iw_i$.
If $W=0$, both cost sums are zero.
Otherwise, let
\[
 c_D=\min\left\{
 t\in\{s_i^*:w_i>0\}:
 \sum_{j:s_j^*\le t}w_j\ge pW
 \right\}.
\]
By \eqref{eq:inclusion} and Proposition~\ref{wo:exact},
\[
 e_i^*=1 \quad\text{if }s_i^*<c_D,
 \qquad
 e_i^*=0 \quad\text{if }s_i^*>c_D
 \quad (w_i>0),
 \qquad
 \sum_iw_ie_i^*=pW.
\]
Since $\sum_iw_ie_i\ge pW$, we obtain
\[
\begin{aligned}
 \sum_iw_is_i^*(e_i-e_i^*)
 &=
 \sum_iw_i(s_i^*-c_D)(e_i-e_i^*)
 +c_D\left(\sum_iw_ie_i-pW\right)\\
 &\ge0.
\end{aligned}
\]
The first term is nonnegative: below the cutoff,
$e_i-e_i^*\le0$; above it, $e_i-e_i^*\ge0$;
at the cutoff, the factor $s_i^*-c_D$ is zero.
The second term is nonnegative because $c_D>0$
and $C$ satisfies the coverage constraint.

Thus $\sum_iw_is_i^*e_i^*\le\sum_iw_is_i^*e_i$
for every $D$. Integrating this inequality using
\eqref{wo:risk-identity} proves optimality, including
when the risks are infinite.
Exact coverage of $C^*$ follows from
Proposition~\ref{wo:exact}.
\end{proof}

{
This is the weighted version of the exchangeable finite-sample oracle;
see \citep{hoff2023,koning2025fuzzy,tibshirani2026testing} for preceding
formulations. At $w\equiv1$, the validity class reduces to ordinary
exchangeability. The weight $w(x)$ enters \eqref{eq:calibrated},
but the optimal score remains $1/f(y\mid x)$.
}

For a positive, finite density estimate $\widehat f$ fitted
independently of the calibration and test data, the rule with
score $1/\widehat f$ still has coverage exactly $p$ under every
$H\in\mathcal H_w$ by Proposition~\ref{wo:exact}.
However, its target risk need not equal $R_m(C^*)$.

\section{Approximating the weighted oracle}\label{sec:approximation}

\begin{theorem}[Integrated and uniform score-error bounds]
\label{wo:regret}
Suppose $R_m(C^*)<\infty$. For any fixed finite score $s$, define
\[
 d_1(s)=\mathbb E_{(X,Y)\sim Q}|s(X,Y)-s^*(X,Y)|,
 \qquad d_\infty(s)=\|s-s^*\|_\infty.
\]
Then
\begin{equation}\label{wo:regret-bounds}
 0\le R_m(C_s)-R_m(C^*)
 \le\min\{d_1(s),\ 2\min(p,\alpha)d_\infty(s)\}.
\end{equation}
The uniform norm can be replaced by an essential uniform norm under
$P_X\otimes\mu$. 
\end{theorem}
\begin{proof}
Fix $D=(z_1,\ldots,z_N)$, and write $W=W(D)$,
$s_i=s(x_i,y_i)$, $e_i^s=e_i^{C_s}(D)$, and
$e_i^*=e_i^{C^*}(D)$.
By Proposition~\ref{wo:exact},
\begin{equation}\label{wo:equal-weighted-sums}
\sum_i w_i e_i^s=\sum_i w_i e_i^*=pW.
\end{equation}

{
The ordering argument in the proof of Theorem~\ref{wo:oracle},
applied to each score with the equal weighted sums in
\eqref{wo:equal-weighted-sums}, gives
\[
 \sum_iw_is_i(e_i^s-e_i^*)\le0,
 \qquad
 \sum_iw_is_i^*(e_i^s-e_i^*)\ge0.
\]
The cutoff term vanishes because those sums agree, so the first
inequality also holds for scores of either sign.
}
Consequently,
\begin{equation}\label{wo:local-error}
\begin{aligned}
0
&\le \sum_i w_i s_i^*(e_i^s-e_i^*)\\
&\le \sum_i w_i(s_i^*-s_i)(e_i^s-e_i^*)\\
&\le \sum_i w_i|s_i-s_i^*|\,|e_i^s-e_i^*|.
\end{aligned}
\end{equation}
If $W=0$, all terms are zero.

Since $|e_i^s-e_i^*|\le1$, taking expectations in
\eqref{wo:local-error} and using \eqref{wo:risk-identity}
and $dQ=w\,dP$ gives
\[
\begin{aligned}
0\le R_m(C_s)-R_m(C^*)
&\le \frac1N\mathbb E_{P^N}
       \sum_i w_i|s_i-s_i^*|\\
&=\mathbb E_Q|s-s^*|
=d_1(s).
\end{aligned}
\]

Also, $0\le e_i^s,e_i^*\le1$ and
\eqref{wo:equal-weighted-sums} imply
\[
\begin{aligned}
\sum_i w_i|e_i^s-e_i^*|
&\le
\min\left\{
\sum_i w_i(e_i^s+e_i^*),
\sum_i w_i(2-e_i^s-e_i^*)
\right\}\\
&=2\min(p,\alpha)W.
\end{aligned}
\]
Using this bound in \eqref{wo:local-error}, together with
$|s_i-s_i^*|\le d_\infty(s)$ and
$\mathbb E_{P^N}W=N$, gives
\[
R_m(C_s)-R_m(C^*)
\le 2\min(p,\alpha)d_\infty(s).
\]

Finally, $P$ has density $f$ with respect to
$P_X\otimes\mu$. Thus a bound on $|s-s^*|$ holding
$(P_X\otimes\mu)$-almost everywhere holds at all $N$
observations $P^N$-almost surely.
The same argument therefore applies to the essential
uniform norm.
\end{proof}

For a $\mu$-integrable function $q:\mathcal Y\to\mathbb R$, define
\begin{equation}\label{eq:l1-norm}
 \|q\|_{L^1(\mu)}
 =\int_{\mathcal Y}|q(y)|\,d\mu(y).
\end{equation}

\begin{corollary}[A density-estimation bound]\label{wo:density-bound}
Let $\widehat f$ be a fixed, positive, finite measurable density estimate,
and use $s=1/\widehat f$ in \eqref{eq:calibrated}. When $R_m(C^*)<\infty$,
\begin{equation}\label{wo:relative-density-error}
 R_m(C_s)-R_m(C^*)
 \le \mathbb E_{Q_X}\int_{\mathcal Y}
              \frac{|\widehat f(y\mid x)-f(y\mid x)|}{\widehat f(y\mid x)}\,d\mu(y).
\end{equation}
If $\widehat f\ge a>0$, the right side is at most
\begin{equation}\label{wo:density-l1}
 \frac1a\mathbb E_{Q_X}
      \|\widehat f(\cdot\mid x)-f(\cdot\mid x)\|_{L^1(\mu)}.
\end{equation}
\end{corollary}
\begin{proof}
Because the conditional density under $Q$ is also $f$,
\[
 \mathbb E_Q\left|\frac1{\widehat f}-\frac1f\right|
 =\mathbb E_{Q_X}\int
       \frac{|\widehat f-f|}{\widehat f}\,d\mu.
\]
Apply Theorem~\ref{wo:regret}; the lower bound on $\widehat f$ gives
\eqref{wo:density-l1}.
\end{proof}

{
\begin{corollary}[Density ordering and oracle-length consistency]
\label{wo:density-ordering}
Assume, as in Theorem~\ref{wo:oracle}, that $f$ is positive and
finite $(P_X\otimes\mu)$-almost everywhere, and suppose
$V=\mu(\mathcal Y)<\infty$.
Let $\widetilde f\ge0$ be a fixed finite measurable density estimate and use
$s=-\widetilde f$ in \eqref{eq:calibrated}. If $\widetilde f>0$,
the score $1/\widetilde f$ gives exactly the same prediction rule.
Define
\[
 e=\mathbb E_{Q_X}
       \|\widetilde f(\cdot\mid X)-f(\cdot\mid X)\|_{L^1(\mu)},
\]
and, for $\tau>0$, define
\begin{equation}\label{eq:density-tail}
 \mathcal T_f(\tau)
 =\mathbb E_{Q_X}\int_{\mathcal Y}
       \left(1-\frac{f(y\mid X)}{\tau}\right)_+\,d\mu(y),
 \qquad u_+=\max\{u,0\}.
\end{equation}
For every calibration size $m\ge1$,
\begin{equation}\label{wo:density-order-bound}
 0\le R_m(C_s)-R_m(C^*)
 \le\inf_{\tau>0}\left\{\frac e\tau+\mathcal T_f(\tau)\right\}.
\end{equation}
The parameter $\tau$ is used only in the bound; the density estimate
is not clipped. In particular, if the true density satisfies
$f\ge c>0$ almost everywhere, then
\begin{equation}\label{wo:density-positive-lower}
 R_m(C_s)-R_m(C^*)\le e/c,
\end{equation}
without a positive lower bound on $\widetilde f$.

Now fix $P,Q,w,p$, and the response scale. For estimates
$\widetilde f_n$ trained independently of calibration and test data,
put $s_n=-\widetilde f_n$ and
\begin{equation}\label{eq:density-mean-error}
 e_n=\mathbb E_{Q_X}
       \|\widetilde f_n(\cdot\mid X)-f(\cdot\mid X)\|_{L^1(\mu)},
 \qquad \bar e_n=\mathbb E_{\rm train}e_n.
\end{equation}
Each risk below is evaluated conditional on training, and each oracle
uses the same calibration size as its fitted comparator. Then
\begin{equation}\label{wo:density-mean-bound}
 \sup_{m\ge1}\mathbb E_{\rm train}
       [R_m(C_{s_n})-R_m(C^*)]
 \le\inf_{\tau>0}
       \left\{\frac{\bar e_n}{\tau}+\mathcal T_f(\tau)\right\}.
\end{equation}
If $\bar e_n\to0$, the right side tends to zero.
If instead $e_n\to0$ in probability over training, then
$\sup_{m\ge1}[R_m(C_{s_n})-R_m(C^*)]\to0$ in probability
and in training expectation.
Conditional on training, each rule has coverage exactly $p$
under every $H\in\mathcal H_w$.
\end{corollary}
\begin{proof}
Finite response volume gives
\begin{equation}\label{eq:finite-volume-risk}
 R_m(C^*)\le\mu(\mathcal Y)=V<\infty.
\end{equation}
Fix $\tau>0$ and $\eta>0$. The function
\[
 \psi_{\tau,\eta}(v)=
 \begin{cases}
   1/v,&v\ge\tau,\\[1mm]
   1/\tau+\eta(\tau-v)/\tau^2,&0\le v<\tau
 \end{cases}
\]
is finite and strictly decreasing on $[0,\infty)$.
Thus $\psi_{\tau,\eta}(\widetilde f)$ and $-\widetilde f$
preserve exactly the same comparisons and ties, and yield the same
rule in \eqref{eq:calibrated} for every calibration sample and $U$.
Write $f=f(y\mid x)$ at a fixed point. The inequality
\[
 \frac{|\max\{\widetilde f,\tau\}-f|}
      {\max\{\widetilde f,\tau\}}
 \le\frac{|\widetilde f-f|}{\tau}
       +\left(1-\frac f\tau\right)_+
\]
and the identity
\[
 \psi_{\tau,\eta}(\widetilde f)
 =\frac1{\max\{\widetilde f,\tau\}}
       +\frac{\eta}{\tau^2}(\tau-\widetilde f)_+
\]
give, because $\int f(y\mid x)\,d\mu(y)=1$,
\[
 \mathbb E_Q\left|
       \psi_{\tau,\eta}(\widetilde f)-\frac1f\right|
 \le \frac e\tau+\mathcal T_f(\tau)+\frac\eta\tau.
\]
Apply Theorem~\ref{wo:regret} and let $\eta\downarrow0$
only in this upper bound. The prediction rule is unchanged for
every $\eta>0$, so this limit does not merge score ties.
Taking the infimum over $\tau$ proves
\eqref{wo:density-order-bound}. If $f\ge c$, choose $\tau=c$;
then $\mathcal T_f(c)=0$, proving
\eqref{wo:density-positive-lower}.

For every fixed $\tau>0$, average the same bound over training.
It is independent of $m$, which proves
\eqref{wo:density-mean-bound}.
Since $f>0$ $(Q_X\otimes\mu)$-almost everywhere,
$\mathcal T_f(\tau)\to0$ as $\tau\downarrow0$ by dominated
convergence on the finite-volume response domain.
For fixed $\tau$, let $n\to\infty$, and then let $\tau\downarrow0$.
This proves the mean convergence when $\bar e_n\to0$.
The deterministic bound, with the same two limits, proves convergence
in probability when $e_n\to0$ in probability. The supremum risk gap
lies in $[0,V]$, so its convergence in probability also implies
convergence in expectation. Exact coverage follows from
Proposition~\ref{wo:exact}.
\end{proof}
}

{
\begin{proposition}[Linear density-error bound at fixed calibration size]
\label{prop:fixed-m-density}
Under the assumptions and notation of
Corollary~\ref{wo:density-ordering}, suppose $0<V<\infty$ and
the estimate is normalized:
\[
 \int_{\mathcal Y}\widetilde f(y\mid x)\,d\mu(y)=1
 \qquad\text{for }P_X\text{-almost every }x.
\]
For every $m\ge1$, the rule with score $s=-\widetilde f$ satisfies
\begin{equation}\label{eq:fixed-m-density-bound}
 0\le R_m(C_s)-R_m(C^*)
 \le\min\left\{V,\ mVe,\
       \inf_{\tau>0}\left(\frac e\tau+\mathcal T_f(\tau)\right)\right\}.
\end{equation}
For independently trained estimates $\widetilde f_n$ that are
normalized almost surely, use $s_n$ and $\bar e_n$ from
\eqref{eq:density-mean-error}. Then
\begin{equation}\label{eq:fixed-m-density-mean}
\begin{aligned}
 0&\le\mathbb E_{\rm train}[R_m(C_{s_n})-R_m(C^*)]\\
  &\le\min\left\{V,\ mV\bar e_n,\
       \inf_{\tau>0}\left(\frac{\bar e_n}{\tau}
                              +\mathcal T_f(\tau)\right)\right\}.
\end{aligned}
\end{equation}
Thus, for fixed $m$ and a deterministic sequence $a_n\downarrow0$,
a density-error rate $\bar e_n=O(a_n)$
gives an excess-length rate $O(a_n)$ without a uniform positive
density lower bound. No upper or positive lower bound on $w$ is
required beyond \eqref{eq:shift}.
\end{proposition}

{
Appendix~\ref{app:fixed-m-density-proof} gives the proof, a sharper
weighted bound, and examples showing that neither bound dominates:
$mVe$ is linear in density error at fixed $m$, whereas the tail bound
is uniform in $m$. Appendix~\ref{app:density-linear-sharpness} gives
smooth normalized density pairs with $m=1$, $p=1/2$, and excess
length $e/3$, ruling out
a uniform $o(e)$ bound. This establishes the error order, not optimal
dependence on $m$ or a statistical minimax rate.
}
}

\section{From the oracle to a fitted procedure}\label{sec:algorithm}

\subsection{Quantile histograms}

{
Take
\begin{equation}\label{eq:response-domain}
\mathcal Y=[a,b],\qquad 0<L=b-a<\infty.
\end{equation}
Let $\mu$ be Lebesgue measure.
}
Fix an integer $B\ge1$.
Using training data independent of the calibration and
test data, obtain endpoints
\[
a=z_0(x)<z_1(x)<\cdots<z_B(x)=b,
\qquad x\in\mathcal X,
\]
where $z_j(x)$ estimates the $j/B$-quantile of the
conditional distribution of $Y$ given $X=x$,
for $j=1,\ldots,B-1$.

Let $I_j(x)=[z_{j-1}(x),z_j(x))$ for $j<B$,
$I_B(x)=[z_{B-1}(x),b]$, and
$\Delta_j(x)=z_j(x)-z_{j-1}(x)$.
Assign probability $1/B$ to each bin and define
\begin{equation}\label{eq:histogram}
\widehat f_B(y\mid x)=\frac{1}{B\Delta_j(x)},
\qquad
h_B(x,y)=B\Delta_j(x),
\qquad y\in I_j(x).
\end{equation}
Define $C_{h_B}$ by \eqref{eq:calibrated} with $s=h_B$.
Since $h_B(x,\cdot)$ is constant on each bin,
$C_{h_B}(t,x,u)$ is a union of entire bins for every
fixed $t,x$ and $u\in[0,1]$.

\subsection{An explicit approximation bound}

Condition on the training data used to construct the endpoints
$z_j(x)$. Let $f$ be the jointly measurable true conditional density,
with $\int_a^b f(y\mid x)\,dy=1$ for every $x$, and define
\begin{equation}\label{eq:true-conditional-cdf}
F_x(y)=\int_a^y f(t\mid x)\,dt.
\end{equation}
Assume that, for constants $c>0$ and $0\le K<\infty$,
\begin{equation}\label{eq:regularity}
f(y\mid x)\ge c,
\qquad
|f(y\mid x)-f(t\mid x)|\le K|y-t|
\end{equation}
for all $x\in\mathcal X$ and $y,t\in[a,b]$. Define
\begin{equation}\label{eq:quantile-error}
\varepsilon
=
\sup_{x\in\mathcal X}\max_{0\le j\le B}
|F_x(z_j(x))-j/B|.
\end{equation}

\begin{theorem}[Quantile error controls excess length]
\label{thm:quantile}
Under \eqref{eq:regularity}, let
\begin{equation}\label{eq:eta}
\eta_B(\varepsilon)
=
\frac{2B\varepsilon}{c}
+
\frac{K(1+2B\varepsilon)^2}{2c^3B},
\end{equation}
where $\varepsilon$ is defined in \eqref{eq:quantile-error}.
Then
\begin{equation}\label{eq:quantile-score-bound}
\|h_B-1/f\|_\infty\le\eta_B(\varepsilon),
\end{equation}
and, for every calibration size $m$,
\begin{equation}\label{eq:quantile-risk-bound}
0\le R_m(C_{h_B})-R_m(C^*)
\le 2\min\{p,\alpha\}\eta_B(\varepsilon).
\end{equation}
\end{theorem}

\begin{proof}
Fix $x$ and a bin $I_j(x)$. Write
\[
\Delta=\Delta_j(x),
\qquad
p_j(x)=F_x(z_j(x))-F_x(z_{j-1}(x)).
\]
By \eqref{eq:quantile-error} and \eqref{eq:regularity},
\[
|p_j(x)-1/B|\le2\varepsilon,
\qquad
c\Delta\le p_j(x)\le1/B+2\varepsilon.
\]
It follows that 
\begin{equation}\label{eq:bin-width-bound}
\Delta\le\frac{1/B+2\varepsilon}{c}.
\end{equation}
For $y\in I_j(x)$, \eqref{eq:regularity} gives
\[
|\Delta f(y\mid x)-p_j(x)|
\le K\int_{z_{j-1}(x)}^{z_j(x)}|y-t|\,dt
\le \frac{K\Delta^2}{2}.
\]
Since $h_B(x,y)=B\Delta$,
\begin{align*}
\left|h_B(x,y)-\frac1{f(y\mid x)}\right|
&\le
\frac{B|\Delta f(y\mid x)-p_j(x)|+|B p_j(x)-1|}{f(y\mid x)}\\
&\le
\frac{BK\Delta^2}{2c}+\frac{2B\varepsilon}{c}
\le\eta_B(\varepsilon).
\end{align*}
The last inequality follows from \eqref{eq:bin-width-bound}
and the definition of $\eta_B(\varepsilon)$ in \eqref{eq:eta}.
Taking the supremum over $x$ and $y$ proves
\eqref{eq:quantile-score-bound}.
By \eqref{eq:finite-volume-risk} with $\mu(\mathcal Y)=L$,
the assumption $R_m(C^*)<\infty$ of
Theorem~\ref{wo:regret} is satisfied.
Applying that theorem with $s=h_B$ and using
\eqref{eq:quantile-score-bound} gives
\eqref{eq:quantile-risk-bound}.
\end{proof}

{
\subsection{Kernel histograms}\label{sec:kde}

Let $\mathcal Y$ and $L$ be as in \eqref{eq:response-domain}, and
let $\phi,\Phi$ be the standard normal density and CDF. For $h>0$,
define the Gaussian kernel normalized on $[a,b]$ by
\begin{equation}\label{eq:kde-kernel}
 Z_h(t)=\Phi((b-t)/h)-\Phi((a-t)/h),\qquad
 k_h(y,t)=\frac{\phi((y-t)/h)}{hZ_h(t)},
 \quad y,t\in[a,b].
\end{equation}
Given training responses $Y_1,\ldots,Y_{N_{\rm train}}$ and
measurable training weights $a_i(x)\ge0$ with $\sum_i a_i(x)=1$,
set
\[
 \widehat g_h(y\mid x)=\sum_{i=1}^{N_{\rm train}}a_i(x)k_h(y,Y_i).
\]
Training is independent of final calibration and test data.
Section~\ref{sec:forest-density} defines the forest weights.

Partition $[a,b]$ into $B$ equal-width bins $J_1,\ldots,J_B$,
including $b$ in the last bin. For $\lambda\in[0,1)$ common to
all $(x,y)$ and $y\in J_j$, define
\begin{equation}\label{eq:kde-histogram}
\begin{aligned}
 \widetilde f_{h,B}(y\mid x)
 &=\frac{B}{L}\int_{J_j}\widehat g_h(u\mid x)\,du,\\
 \widehat f_{h,B,\lambda}(y\mid x)
 &=(1-\lambda)\widetilde f_{h,B}(y\mid x)+\frac{\lambda}{L},\\
 s_{h,B,\lambda}(x,y)&=\frac1{\widehat f_{h,B,\lambda}(y\mid x)}.
\end{aligned}
\end{equation}
Writing $\widehat\pi_j(x)=\int_{J_j}\widehat f_{h,B,\lambda}(u\mid x)\,du$,
the bin score is $L/(B\widehat\pi_j(x))$; these probabilities need
not be equal. The densities are positive even at $\lambda=0$.
Calibration therefore returns a union of entire bins.
{
The common affine density mixture in \eqref{eq:kde-histogram}
preserves all density comparisons and ties. Thus, for every
calibration sample, query, and randomization value,
\begin{equation}\label{eq:uniform-mixture-invariance}
 C_{s_{h,B,\lambda}}=C_{-\widetilde f_{h,B}}.
\end{equation}
}

\begin{proposition}[Kernel estimation and excess length]
\label{prop:kde-error}
Let $G\ge1$ be an integer and suppose
$\mathcal X=\{0,\ldots,G-1\}$.
For each $g\in\mathcal X$, let $Y_{g,1},\ldots,Y_{g,n}$ be iid
training responses drawn from the conditional distribution of $Y$
given $X=g$, with true density $f_g(y)=f(y\mid g)$.
Set
\[
 \widehat g_h(y\mid g)=\frac1n\sum_{i=1}^n k_h(y,Y_{g,i}),
\]
and use \eqref{eq:kde-histogram} to construct
$\widetilde f_{h,B}(\cdot\mid g)$ and
$\widehat f_{h,B,\lambda}(\cdot\mid g)$ as estimates of $f_g$.
Assume \eqref{eq:regularity}, and set
$M_f=\max_{g\in\mathcal X}\sup_{y\in\mathcal Y}f_g(y)$,
which is finite because each $f_g$ is Lipschitz on $[a,b]$
and $\mathcal X$ is finite.
Fix $h>0$, an integer $B\ge1$, and $\lambda\in[0,1)$, and define
\begin{equation}\label{eq:kde-error-budget}
 \mathcal E_{n,h,B}
 =\sqrt{\frac2\pi}(LK+2M_f)h
   +\sqrt{\frac Bn}+\frac{KL^2}{3B}.
\end{equation}
Writing $\mathbb E_{\rm train}$ for expectation over the training data,
\begin{equation}\label{eq:kde-density-error}
\begin{aligned}
 \mathbb E_{\rm train}\mathbb E_{Q_X}
 \|\widetilde f_{h,B}(\cdot\mid X)-f(\cdot\mid X)\|_1
 &\le\mathcal E_{n,h,B},\\
 \mathbb E_{\rm train}\mathbb E_{Q_X}
 \|\widehat f_{h,B,\lambda}(\cdot\mid X)-f(\cdot\mid X)\|_1
 &\le(1-\lambda)\mathcal E_{n,h,B}+2\lambda.
\end{aligned}
\end{equation}
Here $\|\cdot\|_1$ denotes the $L^1(\mu)$ norm defined in
\eqref{eq:l1-norm}.
For every final calibration size $m$,
\begin{equation}\label{eq:kde-length-error}
 0\le
 \mathbb E_{\rm train}\!\left[
 R_m(C_{s_{h,B,\lambda}})-R_m(C^*)\right]
 \le\frac{\mathcal E_{n,h,B}}{c}.
\end{equation}
Conditional on the training data, the rule
$C_{s_{h,B,\lambda}}$ defined by \eqref{eq:calibrated}
has target marginal coverage exactly $p$.
\end{proposition}

The proof is in Appendix~\ref{app:kde-proof}.

For the equal-weight estimators in
Proposition~\ref{prop:kde-error}, taking
$B_n=\lceil n^{1/3}\rceil$ and $h_n=n^{-1/3}$
gives $\mathcal E_{n,h_n,B_n}=O(n^{-1/3})$.
Consequently, for every fixed $m$,
\[
 \mathbb E_{\rm train}\!\left[
 R_m(C_{s_{h_n,B_n,\lambda}})-R_m(C^*)\right]
 =O(n^{-1/3}).
\]
}

{
\subsection{Forest weights and density error}\label{sec:forest-density}

Let $\mathcal D_{\rm train}=\{(X_i,Y_i):1\le i\le n\}$,
where $n=N_{\rm train}$, and fit a regression forest with
$N_{\rm tree}\ge1$ trees. The training sample and fitting
randomization are independent of final calibration and test data;
see Appendix~\ref{sec:real-protocol} for the fitting procedure.
For tree $t$, let $\mathcal L_t(x)$ be the terminal region containing
$x$, and define
\[
 \mathcal I_t(x)=\{i\in\{1,\ldots,n\}:X_i\in\mathcal L_t(x)\}.
\]
Count each original observation once, regardless of bootstrap
multiplicity. Every leaf contains at least one original input, so
\begin{equation}\label{eq:forest-training-weights}
 a_i(x)=\frac1{N_{\rm tree}}
 \sum_{t=1}^{N_{\rm tree}}
 \frac{\mathbf1\{i\in\mathcal I_t(x)\}}{|\mathcal I_t(x)|},
 \qquad i=1,\ldots,n,
\end{equation}
is well defined. Each tree assigns weight $1/|\mathcal I_t(x)|$
to observations in the leaf containing $x$ and zero to the others;
$a_i(x)$ averages these weights over trees. Thus $a_i(x)\ge0$
and $\sum_i a_i(x)=1$.
These training weights estimate the conditional distribution and
are distinct from the covariate-shift ratio $w(x)$.

Using these weights in \eqref{eq:kde-histogram} gives the forest
kernel histogram. With $F_x$ from \eqref{eq:true-conditional-cdf},
define its empirical conditional CDF and error:
\begin{equation}\label{eq:forest-cdf-error}
\begin{aligned}
 \widehat F_n(y\mid x)
 &=\sum_{i=1}^{n}a_i(x)\mathbf 1\{Y_i\le y\},\\
 \Delta_n(x)&=\sup_{y\in\mathcal Y}
       |\widehat F_n(y\mid x)-F_x(y)|,\\
 d_n&=\mathbb E_{\rm train}\mathbb E_{Q_X}\Delta_n(X),
\end{aligned}
\end{equation}
{
The expectation $\mathbb E_{\rm train}$ averages over the training
sample and the forest-fitting randomization described in
Appendix~\ref{sec:real-protocol}.
}

For quantile histograms, choose a permutation $\pi$ of
$\{1,\ldots,n\}$ such that
$Y_{\pi(1)}\le\cdots\le Y_{\pi(n)}$, and define
\[
 W_i(x)=\sum_{r=1}^{i}a_{\pi(r)}(x),
 \qquad i=1,\ldots,n.
\]
In the ordering $\pi$, ties between training responses are
broken by their original observation indices.
For each fixed $x$ and $W_1(x)\le\tau<1$, define
\[
 i_\tau
 =\max\{i\in\{1,\ldots,n\}:W_i(x)\le\tau\}.
\]
Then $i_\tau<n$ and
$W_{i_\tau}(x)\le\tau<W_{i_\tau+1}(x)$.
Define
\[
\widehat q_n^{\rm lin}(\tau\mid x)=
\begin{cases}
 Y_{\pi(1)},
 & \tau<W_1(x),\\[4pt]
 Y_{\pi(i_\tau)}
 +\displaystyle
 \frac{\tau-W_{i_\tau}(x)}
      {W_{i_\tau+1}(x)-W_{i_\tau}(x)}
 \bigl(Y_{\pi(i_\tau+1)}-Y_{\pi(i_\tau)}\bigr),
 & W_1(x)\le\tau<1,\\[6pt]
 Y_{\pi(n)},
 & \tau\ge1.
\end{cases}
\]
The denominator in the middle branch is strictly positive.
At a repeated cumulative weight, this rule takes the last
corresponding response; between distinct cumulative weights,
it interpolates from the last response at the lower weight
to the first response at the upper weight.

For $0<\eta<1$, set
\begin{equation}\label{eq:forest-quantile-endpoints}
\begin{aligned}
 z_0(x)&=a,\qquad z_B(x)=b,\\
 z_j(x)&=(1-\eta)\widehat q_n^{\rm lin}(j/B\mid x)
           +\eta(a+Lj/B),\qquad 1\le j<B.
\end{aligned}
\end{equation}
These endpoints are strictly increasing.
For the bins $I_j(x)$ defined as in \eqref{eq:histogram}
using these endpoints, define

\begin{align}
 \widehat f^{\rm q}_{n,B,\eta}(y\mid x)
 &=\frac{1}{B\{z_j(x)-z_{j-1}(x)\}},\notag\\
 s^{\rm q}_{n,B,\eta}(x,y)
 &=\frac{1}{\widehat f^{\rm q}_{n,B,\eta}(y\mid x)}
  =B\{z_j(x)-z_{j-1}(x)\},
 \qquad y\in I_j(x).\label{eq:forest-quantile-histogram}
\end{align}

{
\begin{proposition}[Density error from forest CDF error]
\label{prop:forest-density}
Let $\mathcal Y=[a,b]$ and $L=b-a$ be as in
\eqref{eq:response-domain}. Keep the oracle positivity assumption
on $f$ and assume the Lipschitz condition in \eqref{eq:regularity},
with constant $K$. Use the forest training weights $a_i(x)$ defined
in \eqref{eq:forest-training-weights} to construct the kernel score
$s_{h,B,\lambda}$ in \eqref{eq:kde-histogram}.
Let $d_n$ be defined in \eqref{eq:forest-cdf-error}, and write
$M_f=\sup_{x,y}f(y\mid x)\le L^{-1}+KL/2$.
For an integer $B\ge1$ and $h>0$, define
\begin{equation}\label{eq:forest-density-budgets}
 \mathcal A^{\rm ker}_{n,h,B}
 =2(B-1)d_n
   +\sqrt{\frac2\pi}(LK+2M_f)h+\frac{KL^2}{3B}.
\end{equation}
Then
\begin{equation}\label{eq:forest-density-bounds}
 \mathbb E_{\rm train}\mathbb E_{Q_X}
 \|\widetilde f_{h,B}(\cdot\mid X)-f(\cdot\mid X)\|_1
 \le\mathcal A^{\rm ker}_{n,h,B},
\end{equation}
and, for every final calibration size $m$,
\begin{equation}\label{eq:forest-kernel-length-bound}
\begin{aligned}
 0&\le\mathbb E_{\rm train}
 [R_m(C_{s_{h,B,\lambda}})-R_m(C^*)]\\
 &\le\min\left\{L,\ mL\mathcal A^{\rm ker}_{n,h,B},\
       \inf_{\tau>0}\left(
       \frac{\mathcal A^{\rm ker}_{n,h,B}}{\tau}
       +\mathcal T_f(\tau)\right)\right\},
\end{aligned}
\end{equation}
where $\mathcal T_f$ is defined in \eqref{eq:density-tail}.
These kernel bounds require no uniform positive lower bound on $f$
and allow any common $\lambda\in[0,1)$.

If, in addition, $f\ge c>0$ as in \eqref{eq:regularity}, use the
same weights to construct $s^{\rm q}_{n,B,\eta}$ in
\eqref{eq:forest-quantile-histogram}, where $0<\eta<1$, and define
\begin{equation}\label{eq:forest-quantile-budgets}
\begin{aligned}
 e_{n,\eta}&=3d_n+M_fL\eta,\\
 \mathcal A^{\rm q}_{n,B,\eta}
 &=2(B-1)e_{n,\eta}
   +\frac{KL}{3c}\left(\frac1B+2e_{n,\eta}\right).
\end{aligned}
\end{equation}
Then
\begin{equation}\label{eq:forest-quantile-density-bound}
 \mathbb E_{\rm train}\mathbb E_{Q_X}
 \|\widehat f^{\rm q}_{n,B,\eta}(\cdot\mid X)-f(\cdot\mid X)\|_1
 \le\mathcal A^{\rm q}_{n,B,\eta}.
\end{equation}
Under this additional lower bound, the length bounds simplify to
\begin{equation}\label{eq:forest-length-bounds}
\begin{aligned}
 0\le\mathbb E_{\rm train}
 [R_m(C_{s_{h,B,\lambda}})-R_m(C^*)]
 &\le\mathcal A^{\rm ker}_{n,h,B}/c,\\
 0\le\mathbb E_{\rm train}
 [R_m(C_{s^{\rm q}_{n,B,\eta}})-R_m(C^*)]
 &\le\mathcal A^{\rm q}_{n,B,\eta}/c.
\end{aligned}
\end{equation}
\end{proposition}
}

The proof is in Appendix~\ref{app:forest-density-proof}.
The same appendix also establishes $d_n\to0$
for fixed $P,Q$ with $Q_X\ll P_X$, under
the forest consistency conditions specified there.
}

\section{Choose the score by calibrated target length}
\label{sec:selection}

Let $M\ge1$ be the number of candidate scores.
Using training data, construct measurable scores
$s_1,\ldots,s_M:\mathcal X\times\mathcal Y\to\mathbb R$
with different parameter choices.

We select a score by estimated prediction-set length and
use independent data for final calibration, following
\citet{yang2025selection}.
Continue to use $\mathcal Y=[a,b]$ and $L=b-a<\infty$.

{
\subsection{Selection from reused data pools}
\label{sec:pool-selection}

Condition on the training data and fitting randomization,
so that the candidate scores $s_1,\ldots,s_M$ are fixed.
Use two independent pools, also independent of training:
{
\begin{equation}\label{eq:selection-pools}
 Z_1^{\rm sel},\ldots,Z_{n_s}^{\rm sel}
 \stackrel{\rm iid}{\sim}P,
 \qquad
 X'_1,\ldots,X'_q
 \stackrel{\rm iid}{\sim}P_X.
\end{equation}
}
Call $Z_1^{\rm sel},\ldots,Z_{n_s}^{\rm sel}$
the labeled pool and $X'_1,\ldots,X'_q$ the input pool.
Given the labeled pool, draw $r\ge1$ independent
bootstrap calibration samples $T_1^*,\ldots,T_r^*$,
each containing $m$ observations sampled uniformly
with replacement from that pool.
Use the input pool to estimate target expected length.

For the selection bound, assume that the density ratio
$w$ satisfies
\begin{equation}\label{eq:pool-weight-bound}
 0<w_-\le w(x)\le w_+<\infty,
 \qquad \kappa=w_+/w_-.
\end{equation}
Using the inclusion probability $r_s$ from
\eqref{eq:inclusion}, estimate each candidate's target
expected length by
\begin{equation}\label{eq:pool-selection}
\begin{aligned}
 \widehat R_k^{\rm pool}
 &=
 \frac{\displaystyle
   \sum_{j=1}^q w(X'_j)
   \left\{\frac1r\sum_{\ell=1}^r
     \int_a^b r_{s_k}(T_\ell^*,X'_j;y)\,dy\right\}}
 {\displaystyle\sum_{j=1}^q w(X'_j)},\\
 \widehat k_{\rm pool}
 &\in\operatorname*{arg\,min}_{1\le k\le M}
       \widehat R_k^{\rm pool},
\end{aligned}
\end{equation}
with a fixed tie rule.
After selection, calibrate the selected score on a new
sample of size $m$ from $P$, independent of training
and selection.

{
\begin{theorem}[Selection from reused data pools]
\label{thm:pool-selection}
Let $\mathcal Y$ and $L$ be as in
\eqref{eq:response-domain}, and assume \eqref{eq:pool-weight-bound}.
Let $\mathbb E_{\rm sel}$ denote expectation over the labeled pool,
the input pool, and the bootstrap draws used to form
$T_1^*,\ldots,T_r^*$. Conditional on training and fitting
randomization, the following bounds hold.

If $n_s\ge m$, define
\begin{equation}\label{eq:pool-repeat-probability}
 \varepsilon_{\rm rep}(n_s,m)
 =1-\prod_{i=0}^{m-1}\left(1-\frac{i}{n_s}\right)
 \le\frac{m(m-1)}{2n_s}
\end{equation}
and
\begin{equation}\label{eq:pool-finite-m-budget}
\begin{aligned}
 \gamma_{n_s,q,m,r,M}
 =2L\bigg\{&\kappa\sqrt{\frac{\log(2M)}{2q}}
       +\sqrt{\frac{\log(2M)}{2r}}\\
       &+\sqrt{\frac{\log(2M)}{2\lfloor n_s/m\rfloor}}
       +\varepsilon_{\rm rep}(n_s,m)\bigg\}.
\end{aligned}
\end{equation}
Then
\begin{equation}\label{eq:pool-selection-bound}
 \mathbb E_{\rm sel}
 \left[
 R_m(C_{s_{\widehat k_{\rm pool}}})
 -\min_{1\le k\le M}R_m(C_{s_k})
 \right]
 \le\gamma_{n_s,q,m,r,M}.
\end{equation}
For deterministic sample sizes and $M$, averaging over training gives
\begin{equation}\label{eq:pool-end-to-end}
\begin{aligned}
 &\mathbb E_{\rm train,sel}
 \left[R_m(C_{s_{\widehat k_{\rm pool}}})-R_m(C^*)\right]\\
 &\quad\le
 \mathbb E_{\rm train}\min_{1\le k\le M}
 \left[R_m(C_{s_k})-R_m(C^*)\right]
 +\gamma_{n_s,q,m,r,M}.
\end{aligned}
\end{equation}
Neither bound requires a positive lower bound on the conditional density.

If, in addition, $f\ge c>0$, define
\begin{equation}\label{eq:pool-error-budget}
 \beta_{n_s,q,m}
 =\frac{L\kappa}{\sqrt q}
  +\frac{16\kappa}{c\sqrt{n_s}}
  +\frac{32\kappa}{c\sqrt m}
  +\frac{4\kappa}{cm}.
\end{equation}
For every $n_s,m,q,r\ge1$, there is also the bound
\begin{equation}\label{eq:pool-growing-m-bound}
 \mathbb E_{\rm sel}
 \left[
 R_m(C_{s_{\widehat k_{\rm pool}}})
 -\min_{1\le k\le M}R_m(C_{s_k})
 \right]
 \le2M\beta_{n_s,q,m}.
\end{equation}
When $c$ is fixed across training samples,
\eqref{eq:pool-end-to-end} also holds with
$\gamma_{n_s,q,m,r,M}$ replaced by $2M\beta_{n_s,q,m}$.
With independent final calibration, the selected rule has target
marginal coverage exactly $p$, conditional on training and selection.
\end{theorem}

The proof is in Appendix~\ref{app:pool-proof}.
The four terms in \eqref{eq:pool-finite-m-budget} account for the
input pool, the finite number of bootstrap samples, the labeled pool,
and repeated indices within a bootstrap sample.
For fixed $m$ and $M$, this bound tends to zero as $n_s,q,r\to\infty$.
The complementary bound \eqref{eq:pool-growing-m-bound} instead
allows fixed $r$ when $n_s,q,m$ grow and $f\ge c$.
Neither bound is uniformly smaller than the other; the latter
remains useful for the experimental allocation $n_s=m$, $r=30$.
}
}

{
\subsection{Oracle-length consistency at fixed calibration size}
\label{sec:forest-fixed-m-consistency}

{
{
\begin{corollary}[Forest kernel histograms at fixed calibration size]
\label{cor:forest-fixed-m-oracle}
Fix $P,Q,w,p$ and $\mathcal Y=[a,b]$, with $L=b-a\in(0,\infty)$.
Assume the oracle positivity condition on $f$ and the Lipschitz
condition in \eqref{eq:regularity}, with $K$ independent of $n$.
Fit the forest weights independently of calibration and test data. Let $d_n$ be as in
\eqref{eq:forest-cdf-error}. For deterministic integers $B_n\ge1$
and bandwidths $h_n>0$, suppose
\begin{equation}\label{eq:forest-fixed-m-growth}
 B_n\to\infty,\qquad h_n\to0,\qquad B_nd_n\to0.
\end{equation}
For deterministic $\lambda_n\in[0,1)$ common to all $(x,y)$,
use \eqref{eq:kde-histogram} to construct
$s_n=s_{h_n,B_n,\lambda_n}$, and calibrate $C_n=C_{s_n}$ on a
fresh sample of fixed size $m\ge1$. Then
\begin{equation}\label{eq:forest-fixed-m-oracle}
 0\le\mathbb E_{\rm train}[R_m(C_n)-R_m(C^*)]
 \le mL\mathcal A^{\rm ker}_{n,h_n,B_n}\longrightarrow0.
\end{equation}
Here $R_m$ is conditional on training, and $C^*$ is the same
finite-$m$ oracle for every $n$. Conditional on training, $C_n$
has target marginal coverage exactly $p$ at every $n$.
\end{corollary}

The proof is in Appendix~\ref{app:forest-fixed-m-proof}.
No uniform positive lower bound on $f$ is required.
The condition $B_nd_n\to0$ limits bin growth relative to CDF error.
}

{
\paragraph{Selection at fixed calibration size.}
\label{sec:forest-fixed-m-selection}
{
The next corollary adds pooled selection, allowing the candidate
library and selection data to grow while the final calibration size
remains fixed.
}

\begin{corollary}[Oracle-length consistency after selection at fixed $m$]
\label{cor:forest-selected-fixed-m}
Fix $P,Q$, the known density ratio $w$, the coverage level $p$, and
$\mathcal Y=[a,b]$ with $L=b-a\in(0,\infty)$.
Assume \eqref{eq:pool-weight-bound}. Suppose $f>0$
$(Q_X\otimes\mu)$-almost everywhere and
$|f(y\mid x)-f(t\mid x)|\le K|y-t|$ on $[a,b]$, with a constant
$K$ independent of the training size $n$.
Let the forest weights be fitted independently of the selection,
final calibration, and test data, and suppose
$d_n\to0$ in \eqref{eq:forest-cdf-error}.

For deterministic integers $J_n\ge1$ and each $j=1,\ldots,J_n$,
include the kernel histogram score
in \eqref{eq:kde-histogram} with
\begin{equation}\label{eq:forest-fixed-m-library}
 B_j=2^j,\qquad h_j=L2^{-j},\qquad \lambda\in[0,1),
\end{equation}
using these weights. Within each candidate, $\lambda$ is common to
all $(x,y)$. Other training-fitted scores may also be included;
let $M_n$ be the deterministic total number of candidates.
For a fixed integer $m\ge1$, use \eqref{eq:pool-selection} with
pool sizes $n_s=n_s(n)$ and $q=q(n)$, and $r=r_n$ bootstrap samples.
If
\begin{equation}\label{eq:forest-selected-fixed-m-growth}
 J_n\longrightarrow\infty,
 \qquad
 \log(2M_n)\left(\frac1{n_s}+\frac1q+\frac1{r_n}\right)
 \longrightarrow0,
\end{equation}
then
\begin{equation}\label{eq:forest-selected-fixed-m}
 \mathbb E_{\rm train,sel}
 \left[R_m(C_{s_{\widehat k_{\rm pool}}})-R_m(C^*)\right]
 \longrightarrow0.
\end{equation}
The comparator is the same finite-$m$ oracle for every $n$.
Independent final calibration gives target marginal coverage exactly
$p$ at every training size, conditional on training and selection.
\end{corollary}

{
The proof is in Appendix~\ref{app:forest-selected-fixed-m-proof}.
No rate for $d_n$ is needed: one permitted choice is
$M_n=O(\log n)$, $n_s,q$ proportional to $n$, and $r_n$ proportional
to $\log n$.
The required $r_n\to\infty$ excludes fixed $r=30$.
Appendix~\ref{sec:forest-selection-consistency} treats growing $m$
with fixed $r$; Appendix~\ref{sec:independent-selection} treats
independent calibration samples.
}
}

\section{Experiments}\label{sec:experiments}

{
The main benchmark uses randomized weighted calibration with target
coverage $p=0.9$. Section~\ref{sec:dcp-shift10} and
Appendices~\ref{app:hpd-low-coverage} and~\ref{sec:coverage-joint}
report additional coverage levels and shift settings.
}

\subsection{Known-density simulations}\label{sec:synthetic}

Calibration data follow $P$, with $P(X=0)=P(X=1)=1/2$;
test data follow $Q$, with $Q(X=0)=1/(1+\rho)$ and
$Q(X=1)=\rho/(1+\rho)$. Both laws have the same conditional
response distribution. Weighted calibration uses
\[
 w(0)=\frac{2}{1+\rho},
 \qquad
 w(1)=\frac{2\rho}{1+\rho}.
\]

Let $f_0(y)$ and $f_1(y)$ denote the conditional densities of $Y$ given
$X=0$ and $X=1$. For $0\le y\le1$, define the beta density
\[
b_{a,b}(y)
=\frac{y^{a-1}(1-y)^{b-1}}
{\displaystyle\int_0^1 t^{a-1}(1-t)^{b-1}\,dt}.
\]
We use the following four pairs of conditional densities.

\begin{center}
\small
\setlength{\tabcolsep}{5pt}
\renewcommand{\arraystretch}{1.1}
\begin{tabular}{@{}lll@{}}
\toprule
Case & $f_0(y)$ & $f_1(y)$ \\
\midrule
Asymmetric unimodal & $0.1+0.9\,b_{2,5}(y)$ & $0.1+0.9\,b_{8,2}(y)$ \\
Heteroscedastic & $0.1+0.9\,b_{2,2}(y)$ & $0.1+0.9\,b_{15,15}(y)$ \\
Separated modes & $0.1+0.9\,b_{2,5}(y)$ & $0.1+0.45\,b_{3,15}(y)+0.45\,b_{15,3}(y)$ \\
Unequal mode weights & $0.1+0.9\,b_{3,7}(y)$ & $0.1+0.25\,b_{3,15}(y)+0.65\,b_{15,3}(y)$ \\
\bottomrule
\end{tabular}
\end{center}

Here $\mathcal{Y}=[0,1]$, and each density has a uniform component
of weight $0.1$.

We use $n$ training observations with $X=0$ and $n$ with $X=1$, where $n\in\{500,2000,8000\}$. We set $\rho=3$ and $m=200$. The seven candidate conditional-density estimates are four
empirical-quantile histograms with $B\in\{10,20,40,80\}$ bins
and three kernel histograms with bandwidths
$h\in\{0.01,0.025,0.05\}$, each using 80 equal-width bins
on $[0,1]$. Further implementation details are given in
Appendix~\ref{app:synthetic}.

{
For kernel estimation of $f(y\mid g)$ at $x=g\in\{0,1\}$,
use $a_i(g)=\1\{X_i=g\}/n$: weight $1/n$ within group $g$
and zero outside it.
}

For length selection, calibrate each fixed candidate on 30 independent
samples of size $m$ from $P$. For each sample,
\eqref{eq:inclusion}--\eqref{eq:integrated-length} give the lengths
at $X=0$ and $X=1$, averaged over $U$. Weight these by $1/4$ and
$3/4$, their target probabilities at $\rho=3$, then average over the
30 samples and select the smallest mean. Integration over $X$ and
$U$ is exact; the 30 samples approximate the calibration expectation.

For fitted densities $\widehat f_k$, $k=1,\ldots,7$, and the same
samples $T_\ell=\{(X_{\ell i},Y_{\ell i})\}_{i=1}^{m}$,
$\ell=1,\ldots,30$, negative log-likelihood (NLL) selection uses
{
\begin{equation}\label{eq:nll-selection}
\begin{aligned}
 \widehat{\mathrm{NLL}}_k
 &=-\frac{1}{30m}\sum_{\ell=1}^{30}\sum_{i=1}^{m}
 w(X_{\ell i})\log\widehat f_k(Y_{\ell i}\mid X_{\ell i}),
 \qquad k=1,\ldots,7,\\
 \widehat k_{\mathrm{NLL}}
 &\in\operatorname*{arg\,min}_{1\le k\le7}
 \widehat{\mathrm{NLL}}_k.
\end{aligned}
\end{equation}
}
NLL evaluates density fit without temporary calibration. Both selectors
break ties by the smallest index and calibrate the selected reciprocal
density score on the same fresh sample of $m$ observations from $P$.
Training, selection, and final calibration are independent.

For all four cases, we also vary $\rho\in\{1,3,9\}$ at $n=2000$ and $m=200$. For the asymmetric and separated-mode cases, we vary $m\in\{50,200,800\}$ at $n=2000$ and $\rho=3$. This gives 24 distinct configurations, each repeated independently 50 times. Within each repetition, all methods share the same data.
Across training sizes, training samples are nested, while
the selection and final calibration samples remain unchanged.

{
The oracle $C^*$ uses the true density and the same final calibration
sample. We integrate length and coverage over $Q$ and randomization,
then report means and standard errors over 50 repetitions, with paired
oracle differences. Appendix~\ref{app:synthetic} gives numerical details.
}
At $n=8000$, the primary selector's paired oracle gaps are
$0.00439$ (SE $0.00138$), $-0.00030$ ($0.00121$),
$0.00298$ ($0.00180$), and $0.00164$ ($0.00251$)
(Table~\ref{tab:synthetic-length}; Figure~\ref{fig:oracle-gap}).
All four gaps are smaller at $n=8000$ than at $n=500$, without
necessarily decreasing at intermediate sizes. Fixed bin counts and
bandwidths do not ensure oracle convergence. Negative gaps can arise from
variation across the 50 repetitions: oracle optimality averages over
calibration samples and need not hold for each realized sample.

For separated and unequal modes, our selector's mean lengths are
$0.5746$ and $0.5583$, versus $0.7920$ and $0.7647$ for DCP-opt.
Our sets average $1.762$ and $1.776$ disjoint intervals; length excludes
the gaps. Gao-style two-component lengths are $0.5804$ and $0.5676$,
and HPD is competitive. NLL selection is slightly shorter in all four
$n=8000$ settings. 

\begin{table}[!htbp]
\centering\small
\setlength{\tabcolsep}{4pt}
\begin{tabular}{@{}lrrrr@{}}
\toprule
Method & Asymmetric & Heterosc. & Separated & Unequal \\
\midrule
Known-density oracle & 0.4603 (0.0060) & 0.4663 (0.0055) & 0.5716 (0.0066) & 0.5566 (0.0067) \\
Density: length selection & 0.4647 (0.0062) & 0.4660 (0.0054) & 0.5746 (0.0069) & 0.5583 (0.0064) \\
Density: NLL selection & 0.4639 (0.0060) & 0.4644 (0.0055) & 0.5740 (0.0066) & 0.5563 (0.0066) \\
Density: one tuning bag & 0.4687 (0.0080) & 0.4709 (0.0063) & 0.5770 (0.0069) & 0.5649 (0.0073) \\
CTI: fixed $B=20$ & 0.4702 (0.0060) & 0.4726 (0.0072) & 0.6092 (0.0076) & 0.5984 (0.0078) \\
Density: quantiles only & 0.4645 (0.0060) & 0.4667 (0.0058) & 0.5766 (0.0068) & 0.5686 (0.0069) \\
Density: kernels only & 0.4633 (0.0060) & 0.4652 (0.0055) & 0.5749 (0.0068) & 0.5560 (0.0067) \\
HPD score & 0.4627 (0.0061) & 0.4673 (0.0058) & 0.5780 (0.0070) & 0.5578 (0.0065) \\
PIT score & 0.5208 (0.0076) & 0.4676 (0.0062) & 0.8006 (0.0027) & 0.7698 (0.0029) \\
DCP-opt score & 0.4636 (0.0058) & 0.4682 (0.0060) & 0.7920 (0.0026) & 0.7647 (0.0032) \\
CQR score & 0.5381 (0.0034) & 0.4670 (0.0056) & 0.8052 (0.0022) & 0.7775 (0.0024) \\
Gao-style, $k=2$ & 0.4631 (0.0062) & 0.4692 (0.0065) & 0.5804 (0.0070) & 0.5676 (0.0068) \\
\bottomrule
\end{tabular}
\caption{Synthetic expected lengths at $n=8{,}000$ {training} labels per group, $m=200$, ratio $3$, and $p=0.9$. Parentheses are standard errors across 50 independent repetitions. The oracle is the exact finite-$m$ weighted density rule, not the population highest-density set. All fitted methods use the same data within a repetition.}
\label{tab:synthetic-length}
\end{table}

\begin{figure}[!htbp]
\centering
\includegraphics[width=\linewidth]{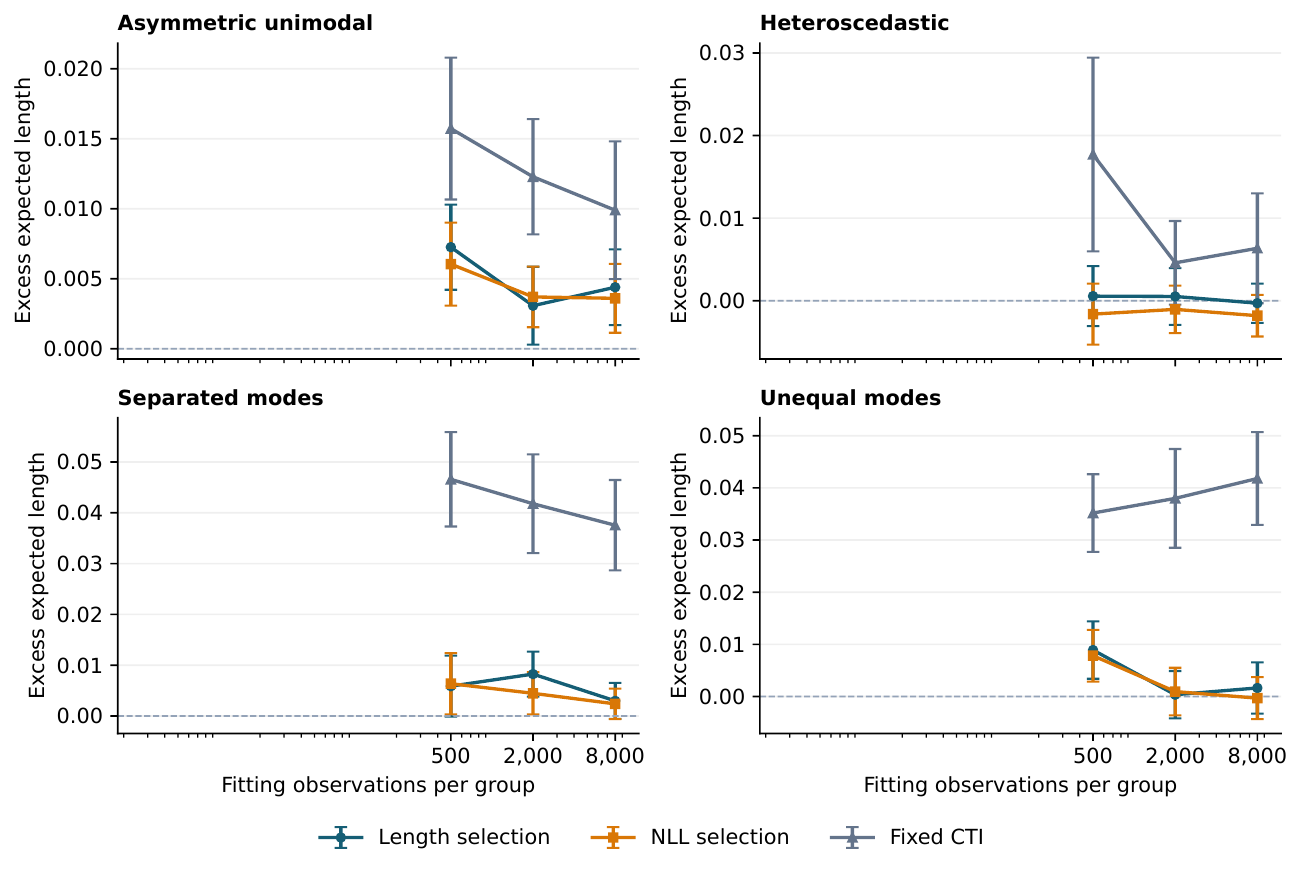}
\caption{Excess expected length above the weighted finite-$m$ oracle
at $m=200$ and shift ratio $3$. Bars are $1.96$ times the standard error
of the paired oracle difference, computed from 50 independent
repetitions. The zero line denotes oracle risk. {Training} observations
are counted per group; design and final calibration sizes stay fixed.}
\label{fig:oracle-gap}
\end{figure}

{
\subsection{Fixed calibration with densities approaching zero}
\label{sec:vanishing-density}

We remove the uniform component from two cases in
Section~\ref{sec:synthetic}:
\[
\begin{array}{lll}
\text{Asymmetric:}&f_0=b_{2,5},&f_1=b_{8,2},\\[2pt]
\text{Separated:}&f_0=b_{2,5},&f_1=\tfrac12b_{3,15}+\tfrac12b_{15,3}.
\end{array}
\]
These Lipschitz densities are positive on $(0,1)$ but vanish at both
endpoints. Fix $P(X=1)=1/2$, $Q(X=1)=3/4$, $p=0.9$, and $m=200$.
{
For $n\in\{250,500,1000,2000,4000,8000\}$, use $n$ training
responses per group (total $2n$), the two independent selection
pools in \eqref{eq:selection-pools} with $n_s=q=n$,
independent of training, and $r_n=\lceil n/25\rceil$
bootstrap calibration samples of size $m$ drawn from the
labeled pool.
} Selection uses
\eqref{eq:pool-selection}, including the observed input-pool group
counts and self-normalized weights $(1,3)$.

Using weight $1/n$ within each input group, construct
\[
J_n=\left\lceil\frac{\log_2n}{3}\right\rceil+2
\]
kernel-histogram candidates as in Section~\ref{sec:kde}, with
\[
B_j=2^j,\qquad h_j=2^{-j},\qquad 1\le j\le J_n.
\]
Set the uniform-mixture weight in \eqref{eq:kde-histogram}
to $\lambda=10^{-6}$ for every candidate.

{
The within-group CDF error $d_n$ in \eqref{eq:forest-cdf-error}
satisfies $d_n=O(n^{-1/2})$ by DKW--Massart
\citep{massart1990dkw}; see \eqref{eq:group-cdf-rate}.
For $j_n=\lceil(\log_2n)/3\rceil$,
\[
n^{1/3}\le B_{j_n}<2n^{1/3},
\qquad h_{j_n}=B_{j_n}^{-1}\to0,
\qquad B_{j_n}d_n=O(n^{-1/6})\to0.
\]
Thus these candidates satisfy
\eqref{eq:forest-fixed-m-growth}.
}
With $M_n=J_n$, $n_s=q=n$, and $r_n=\lceil n/25\rceil$,
the selection condition
\eqref{eq:forest-selected-fixed-m-growth}
in Corollary~\ref{cor:forest-selected-fixed-m} also holds.
The proof in Appendix~\ref{app:forest-selected-fixed-m-proof}
applies to the within-group weights used here.

We include two controls: selection from the same library and pools
with only $r=30$ bootstrap samples, and direct use of candidate
$j_n$ without selection.

{
We use 50 independent repetitions, nested training/pool prefixes,
and shared initial bootstrap draws. All sizes and methods share final
calibration and the finite-$m$ oracle; evaluation follows
Section~\ref{sec:synthetic}. Serial one-thread times include fitting,
selection, calibration, and prediction volume, excluding data generation,
true-density coverage/oracle integration, and output. Selectors fit the
whole library; the prescribed method fits only $j_n$, without selection.
}

{
Figure~\ref{fig:vanishing-density} shows gaps decreasing
} from $0.00745$ to $-0.00108$ and from $0.01381$
to $0.00090$. {Both final gaps are within $1.96$ standard errors of zero.}
Final coverages are $0.9020$ (SE $0.0034$) and $0.9083$ ($0.0029$),
versus oracle coverages $0.9034$ ($0.0033$) and $0.9085$ ($0.0028$).
At $n=8000$, growing-$r_n$ selection takes $0.228/0.230$ seconds,
versus $0.195/0.201$ for $r=30$. Its paired length differences from
$r=30$ are $-0.00057$ (SE $0.00030$) and $0.00066$ ($0.00050$),
with no clear advantage. The prescribed estimator takes $0.025/0.026$
seconds, with mean oracle gaps $0.00500/0.00258$.
{
The complete numerical table is retained in the supplementary result files.
}

\begin{figure}[!htbp]

\centering
\includegraphics[width=\linewidth]{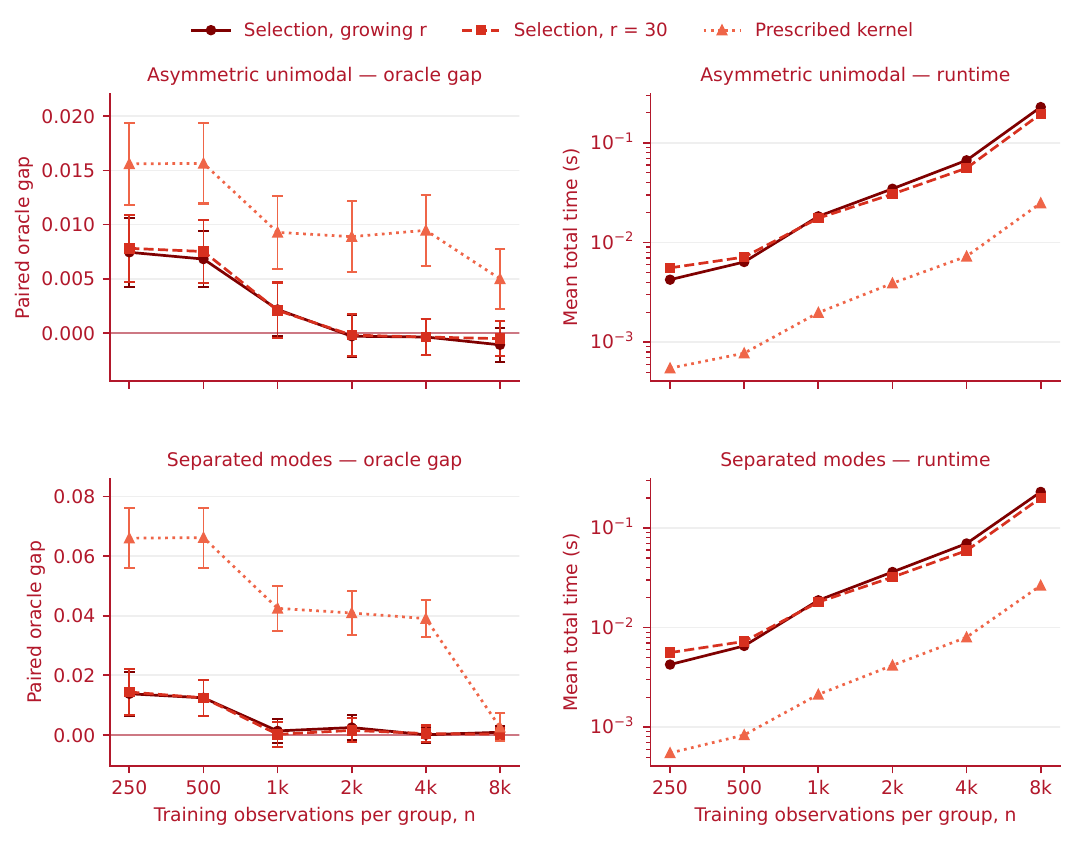}
\caption{{Fixed-$m$ comparison of growing-$r_n$ and
$r=30$ selection with the prescribed estimator (no selection).
Left: paired oracle gaps with $1.96$ Monte Carlo standard-error bars
over 50 repetitions. Right: mean total time on a logarithmic scale.
The design and timing scope are specified in the text.}}
\label{fig:vanishing-density}
\end{figure}
}

\subsection{Public regression datasets under designed shift}\label{sec:real-data}

{
Four public datasets use 30 shared splits, allocating approximately
$40/15/10/15/20\%$ to training, the labeled pool, the input pool,
final calibration, and testing; the pools are defined in
Section~\ref{sec:pool-selection}. Lengths use the transformed $[0,1]$
scale in \eqref{eq:response-sigmoid}, with training-fitted $c,a$.
Relative shift weights are $(1,3)$, with $(1,1)$ as a control;
Section~\ref{sec:dcp-shift10} uses $(1,10)$.
Appendix~\ref{sec:real-protocol} gives the full protocol.
}

{
The primary method selects among 12 forest-weighted quantile and
kernel histograms. HPD, PIT, DCP-opt, and CQR use the same fits.
Ablations change the selection criterion, use fixed CTI, or restrict
the histogram family. Other baselines are reference-classification
density estimation (WQ) \citep{wang2025density}, two-interval DP
\citep{gao2025volume}, and fixed or tuned Ad-EffOrt
\citep{lebars2025volume}.
}

{
Length selection uses \eqref{eq:pool-selection} with $r=30$:
labels from the $15\%$ pool and inputs from the $10\%$ pool.
NLL uses \eqref{eq:nll-selection} on labels from both pools ($25\%$).
Candidate counts are 12 for the full histogram library, two for
selected WQ, six for DP, and three for tuned Ad-EffOrt.
}

\begin{table}[!htbp]
\centering\small
\setlength{\tabcolsep}{4pt}
\renewcommand{\arraystretch}{1.08}
\begin{tabular}{@{}lrrrr@{}}
\toprule
Method & Power Plant & Concrete & Real Estate & Yacht \\
\midrule
Density: length selection & \shortstack[r]{0.0912 (0.0015)\\0.8998 (0.0021)} & \shortstack[r]{0.2303 (0.0037)\\0.9002 (0.0055)} & \shortstack[r]{0.2345 (0.0047)\\0.9026 (0.0071)} & \shortstack[r]{0.0870 (0.0045)\\0.9009 (0.0134)} \\
\addlinespace[2pt]
Density: NLL selection & \shortstack[r]{0.0893 (0.0005)\\0.8973 (0.0021)} & \shortstack[r]{0.2354 (0.0039)\\0.9024 (0.0052)} & \shortstack[r]{0.2414 (0.0058)\\0.9101 (0.0060)} & \shortstack[r]{0.0937 (0.0058)\\0.8947 (0.0159)} \\
\addlinespace[2pt]
Density: one tuning bag & \shortstack[r]{0.0913 (0.0015)\\0.9000 (0.0021)} & \shortstack[r]{0.2289 (0.0034)\\0.8997 (0.0053)} & \shortstack[r]{0.2382 (0.0047)\\0.9027 (0.0073)} & \shortstack[r]{0.0915 (0.0048)\\0.9003 (0.0136)} \\
\addlinespace[2pt]
CTI: fixed QRF & \shortstack[r]{0.1023 (0.0005)\\0.8994 (0.0022)} & \shortstack[r]{0.3088 (0.0044)\\0.8954 (0.0061)} & \shortstack[r]{0.3056 (0.0086)\\0.9003 (0.0083)} & \shortstack[r]{0.2112 (0.0111)\\0.8967 (0.0172)} \\
\addlinespace[2pt]
Density: quantiles only & \shortstack[r]{0.0978 (0.0019)\\0.9013 (0.0022)} & \shortstack[r]{0.2429 (0.0036)\\0.8961 (0.0050)} & \shortstack[r]{0.2943 (0.0185)\\0.8975 (0.0089)} & \shortstack[r]{0.1008 (0.0065)\\0.8899 (0.0157)} \\
\addlinespace[2pt]
Density: kernels only & \shortstack[r]{0.0893 (0.0005)\\0.8973 (0.0021)} & \shortstack[r]{0.2310 (0.0040)\\0.9015 (0.0054)} & \shortstack[r]{0.2361 (0.0045)\\0.9037 (0.0066)} & \shortstack[r]{0.0821 (0.0036)\\0.9029 (0.0141)} \\
\addlinespace[2pt]
HPD score & \shortstack[r]{0.0891 (0.0005)\\0.8974 (0.0021)} & \shortstack[r]{0.2195 (0.0039)\\0.9011 (0.0051)} & \shortstack[r]{0.2282 (0.0050)\\0.8990 (0.0078)} & \shortstack[r]{0.0729 (0.0035)\\0.8981 (0.0168)} \\
\addlinespace[2pt]
PIT score & \shortstack[r]{0.0892 (0.0006)\\0.8984 (0.0023)} & \shortstack[r]{0.2183 (0.0037)\\0.8961 (0.0060)} & \shortstack[r]{0.2277 (0.0054)\\0.8988 (0.0084)} & \shortstack[r]{0.1032 (0.0073)\\0.8864 (0.0145)} \\
\addlinespace[2pt]
DCP-opt score & \shortstack[r]{0.0883 (0.0005)\\0.8980 (0.0022)} & \shortstack[r]{0.2132 (0.0038)\\0.8956 (0.0060)} & \shortstack[r]{0.2314 (0.0069)\\0.8961 (0.0078)} & \shortstack[r]{0.0910 (0.0047)\\0.8865 (0.0148)} \\
\addlinespace[2pt]
CQR score & \shortstack[r]{0.0902 (0.0006)\\0.8967 (0.0022)} & \shortstack[r]{0.2346 (0.0036)\\0.8993 (0.0066)} & \shortstack[r]{0.2384 (0.0047)\\0.8961 (0.0082)} & \shortstack[r]{0.1942 (0.0088)\\0.9017 (0.0112)} \\
\addlinespace[2pt]
{WQ: fixed} & \shortstack[r]{{0.1046 (0.0006)}\\{0.8987 (0.0022)}} & \shortstack[r]{{0.3072 (0.0069)}\\{0.8905 (0.0069)}} & \shortstack[r]{{0.2872 (0.0083)}\\{0.8959 (0.0103)}} & \shortstack[r]{{0.2164 (0.0106)}\\{0.8765 (0.0189)}} \\
\addlinespace[2pt]
{WQ: selected} & \shortstack[r]{{0.1046 (0.0006)}\\{0.8987 (0.0022)}} & \shortstack[r]{{0.3072 (0.0069)}\\{0.8905 (0.0069)}} & \shortstack[r]{{0.2963 (0.0079)}\\{0.9009 (0.0093)}} & \shortstack[r]{{0.2164 (0.0106)}\\{0.8765 (0.0189)}} \\
\addlinespace[2pt]
Gao-style, $k=2$ & \shortstack[r]{0.0914 (0.0006)\\0.9013 (0.0024)} & \shortstack[r]{0.2311 (0.0052)\\0.8987 (0.0047)} & \shortstack[r]{0.2893 (0.0172)\\0.8972 (0.0090)} & \shortstack[r]{0.0951 (0.0058)\\0.9111 (0.0124)} \\
\addlinespace[2pt]
Ad-EffOrt: fixed & \shortstack[r]{0.1033 (0.0005)\\0.8987 (0.0019)} & \shortstack[r]{0.2233 (0.0038)\\0.9046 (0.0048)} & \shortstack[r]{0.2754 (0.0068)\\0.9013 (0.0095)} & \shortstack[r]{0.1126 (0.0045)\\0.9063 (0.0127)} \\
\addlinespace[2pt]
Ad-EffOrt: tuned & \shortstack[r]{0.1033 (0.0005)\\0.8987 (0.0019)} & \shortstack[r]{0.2130 (0.0037)\\0.9016 (0.0063)} & \shortstack[r]{0.2704 (0.0073)\\0.9037 (0.0084)} & \shortstack[r]{0.1047 (0.0040)\\0.9039 (0.0142)} \\
\bottomrule
\end{tabular}
\caption{Main real-data results at shift ratio $3$ and nominal coverage $0.9$. Each cell gives expected length on the first line and coverage on the second line, both as mean (split standard error) over 30 shared splits. Coverage is a probability. Length is measured on the training-transformed $[0,1]$ scale. Target averages use self-normalized importance weights. All rows use the same randomized weighted final calibration; Gao-style and Ad-EffOrt are the stated weighted adaptations.}
\label{tab:real-length}
\label{tab:real-coverage}
\end{table}

\paragraph{Length comparisons.}
{
The primary selector is shorter than fixed CTI, choosing kernels in
23, 26, 26, and 21 shifted splits on Power Plant, Concrete, Real Estate,
and Yacht. This comparison changes both estimation and selection;
identical quantile bins give CTI and inverse-density scores the same ordering.

Against NLL selection from the same 12 fits, its paired length change
is $+0.00186$ (SE $0.00132$), $-0.00512$ ($0.00242$),
$-0.00692$ ($0.00436$), and $-0.00675$ ($0.00365$).
Thirty tuning bags do not improve on one in every dataset.
Both WQ variants give longer mean sets. They coincide on Power Plant,
Concrete, and Yacht because all splits select the fixed leaf size;
Real Estate selects it in 28 splits.
HPD is shorter than the primary selector on all four datasets,
DCP-opt on the first three, tuned Ad-EffOrt on Concrete, and
kernel-only selection on Power Plant and Yacht. True-density oracle
optimality therefore does not imply that the fitted selector is shortest.
}

\paragraph{Coverage and geometry.}
{
{
Tables~\ref{tab:no-shift-length} and~\ref{tab:real-groups} give the
no-shift control and group coverages. The primary selector averages
1.005--1.183 positive-length components; supplementary files give
complete counts.
} Its group coverages are $0.8403/0.9238$ on Real Estate and
$1.0000/0.8708$ on Yacht; HPD also shows substantial group differences.
The theorem constrains marginal coverage.

Across 120 main-benchmark dataset--shift--method cells, mean coverage
ranges from $0.8733$ to $0.9111$; the lowest is no-shift Yacht PIT
(SE $0.0138$). All splits are retained. Standard errors measure
within-dataset split variation, not variation across datasets.
}

{
\subsection{Separating fitted-score and selection effects}
\label{sec:score-selection-diagnostic}

{
Table~\ref{tab:score-selection-diagnostic} varies either the score at
one fitted density or selection at a fixed score. NLL15 and length
selection share the $15\%$ labeled pool; main-table NLL uses $25\%$.
Fits, final calibration, and testing are shared. This exploratory
follow-up reuses the original 30 splits; see
Appendix~\ref{app:diagnostic-protocol}.
}

\begin{table}[!htbp]
\centering\small

\setlength{\tabcolsep}{7pt}
\begin{tabular}{@{}lrr@{}}
\toprule
\multicolumn{3}{@{}l}{\emph{A. Common NLL15 fit: change the score}}\\
Dataset & HPD $-$ inverse density & DCP-opt $-$ inverse density \\
\midrule
Power Plant & $-0.00025$ (0.00025) & $-0.00137$ (0.00029) \\
Concrete & $-0.00173$ (0.00239) & $-0.01806$ (0.00318) \\
Real Estate & $-0.00178$ (0.00332) & $-0.00978$ (0.00471) \\
Yacht & $+0.00375$ (0.00290) & $+0.00477$ (0.00661) \\
\addlinespace
\multicolumn{3}{@{}l}{\emph{B. Fixed score: length selection $-$ NLL15 selection}}\\
Dataset & Inverse-density score & HPD score \\
\midrule
Power Plant & $+0.00186$ (0.00132) & $0.00000$ (0.00000) \\
Concrete & $-0.00512$ (0.00242) & $-0.01414$ (0.00251) \\
Real Estate & $-0.00852$ (0.00462) & $-0.01302$ (0.00478) \\
Yacht & $-0.00871$ (0.00406) & $-0.02656$ (0.00538) \\
\bottomrule
\end{tabular}
\caption{Controlled length differences at shift ratio $3$.
Entries are paired means (standard errors) over 30 splits, on the
transformed $[0,1]$ response scale. Negative values favor the first
method in each difference. Panel A holds the selected density fixed;
Panel B holds the score and the labeled tuning budget fixed.
Uncertainty describes split variation within each dataset.}
\label{tab:score-selection-diagnostic}
\end{table}

{
At the common NLL15 fit, all four descriptive $95\%$ intervals for
HPD minus inverse density contain zero, without establishing equivalence.
At the fixed kernel fit, HPD is shorter on Power Plant and Yacht;
directions vary across the library. For HPD, length selection reduces
mean length by about 6\%, 5\%, and 27\% on Concrete, Real Estate,
and Yacht relative to NLL15.
}

{
On Yacht, HPD's $0.01410$ mean advantage splits algebraically into
$0.00411$ from changing the score at the inverse-density-selected fit
and $0.00999$ from switching fits. This is not a causal attribution.
HPD can change pooled score ordering through its input-dependent
transformation; the comparison does not establish general robustness
to density-estimation error.
}

In six known-density scenarios, $50$ paired repetitions did not resolve
a difference between true-density and true-HPD lengths
(Table~\ref{tab:diagnostic-true-scores}); this does not establish equality.

}

{
A supplementary joint inverse-density/HPD library reduces computation
by constructing fewer histograms but does not consistently shorten
sets relative to HPD (Appendix~\ref{sec:coverage-joint}).
}

{
\subsection{Inverse density and DCP-opt under stronger shift}
\label{sec:dcp-shift10}

Table~\ref{tab:dcp-shift10} compares the inverse-density selector
with DCP-opt at relative shift weights $(1,10)$ and
$p\in\{0.6,0.7,0.8,0.9\}$. Both methods follow Section~\ref{sec:real-data} with the same
12 fitted densities. DCP-opt interval centers, length-based candidate
selection, and independent final calibration are recomputed at each $p$.

\begin{table}[!htbp]
\centering\small
\setlength{\tabcolsep}{3pt}
\begin{tabular}{@{}llrrrrrr@{}}
\toprule
& & \multicolumn{2}{c}{Length (SE)}
& \multicolumn{2}{c}{Coverage (\%)}
& \multicolumn{2}{c}{Seconds} \\
\cmidrule(lr){3-4}\cmidrule(lr){5-6}\cmidrule(l){7-8}
Dataset & $p$ & Density & DCP-opt & Density & DCP-opt & Density & DCP-opt \\
\midrule
Power Plant & 0.6 & 0.04290 (0.00030) & 0.04394 (0.00027) & 60.39 & 60.47 & 5.151 & 6.453 \\
 & 0.7 & 0.05395 (0.00041) & 0.05472 (0.00031) & 69.87 & 69.87 & 5.157 & 6.470 \\
 & 0.8 & 0.06894 (0.00046) & 0.06856 (0.00048) & 79.91 & 79.72 & 5.162 & 6.497 \\
 & 0.9 & 0.09242 (0.00212) & 0.08914 (0.00058) & 89.97 & 89.77 & 5.161 & 6.453 \\
\addlinespace
Concrete & 0.6 & 0.10501 (0.00220) & 0.10581 (0.00225) & 59.82 & 59.01 & 0.299 & 0.503 \\
 & 0.7 & 0.13502 (0.00247) & 0.13170 (0.00250) & 69.35 & 68.84 & 0.299 & 0.499 \\
 & 0.8 & 0.17736 (0.00339) & 0.16557 (0.00270) & 80.04 & 79.53 & 0.298 & 0.495 \\
 & 0.9 & 0.23208 (0.00371) & 0.21865 (0.00417) & 90.02 & 89.87 & 0.299 & 0.476 \\
\addlinespace
Real Estate & 0.6 & 0.09932 (0.00380) & 0.10257 (0.00386) & 59.81 & 60.28 & 0.166 & 0.291 \\
 & 0.7 & 0.12690 (0.00442) & 0.12843 (0.00479) & 69.39 & 69.72 & 0.166 & 0.284 \\
 & 0.8 & 0.16706 (0.00535) & 0.15387 (0.00387) & 81.10 & 80.03 & 0.166 & 0.280 \\
 & 0.9 & 0.22761 (0.00715) & 0.22804 (0.01009) & 90.52 & 89.56 & 0.166 & 0.268 \\
\addlinespace
Yacht & 0.6 & 0.03785 (0.00199) & 0.04201 (0.00223) & 57.55 & 56.96 & 0.159 & 0.263 \\
 & 0.7 & 0.05283 (0.00302) & 0.05677 (0.00279) & 68.54 & 67.40 & 0.159 & 0.254 \\
 & 0.8 & 0.07230 (0.00537) & 0.07538 (0.00292) & 81.23 & 79.80 & 0.158 & 0.250 \\
 & 0.9 & 0.09364 (0.00551) & 0.11366 (0.00715) & 90.39 & 90.20 & 0.158 & 0.235 \\
\bottomrule
\end{tabular}
\caption{Inverse-density and DCP-opt length selection at shift ratio
$10$: means over 30 paired splits, with length standard errors in
parentheses. Length is measured on the training-transformed $[0,1]$
scale. Timing is from a separate serial, one-thread run and includes
the full fitting, selection, and final calibration/evaluation costs
defined in Section~\ref{sec:runtime}. Compare times within this table.
Coverage and timing standard errors and paired intervals are supplied
with the reproduction files.}
\label{tab:dcp-shift10}
\end{table}

Inverse density gives shorter mean sets in 10 of 16 combinations,
including all datasets at $p=0.6$; DCP-opt is shorter in the other six.
Unadjusted $95\%$ paired intervals favor each method in three cases
and contain zero in ten. On Yacht at $p=0.9$, lengths are $0.09364$
and $0.11366$, with coverages $90.39\%$ and $90.20\%$.
Inverse density is faster in all 16 cases; DCP-opt takes
$1.25$--$1.76$ times as long. Length comparisons must also account
for differences in achieved coverage.
}

\subsection{Computational cost}\label{sec:runtime}

{
{
Table~\ref{tab:runtime} times the unchanged $p=0.9$, ratio-$3$
benchmark in Table~\ref{tab:real-length}. Panel A uses four split workers,
each with one numerical thread; Panel B uses serial one-thread runs
and alternating method order. Both reproduce selections and predictions.
Compare times within panels; the comparisons in
Table~\ref{tab:dcp-shift10} and Appendix~\ref{sec:coverage-joint}
use separate serial runs.
}
}

\paragraph{Timing protocol.}
{
Times include fitting/query setup, histograms, scores, selection,
final calibration, and full-test length, coverage, and component evaluation
integrated over randomization. They exclude loading, splitting,
transformation, weights, imports, output, and validation. Shared work is
charged in full to each method. Totals sum measured stages; they are
not cold-start times, single-input latency, or summable wall times.
The reproduction files record versions and stages, but not the processor.

Panel A charges density, NLL, HPD, and DCP-opt for two forests and
12 histograms; Gao for two forests and six quantile histograms;
and tuned Ad-EffOrt for three configurations. Selection includes NLL or
30 resamples; scores include HPD ordering, DCP centers, or Gao DP.
Fixed CTI uses one forest and histogram, querying only calibration
and test inputs.
}

\begin{table}[!htbp]
\centering\small
\setlength{\tabcolsep}{4pt}
\begin{tabular}{@{}lrrrr@{}}
\toprule
\multicolumn{5}{@{}l}{\emph{A. Four split workers; one numerical thread per worker}}\\
Method & Power Plant & Concrete & Real Estate & Yacht \\
\midrule
Density: length selection & 5.7861 (0.0680) & 0.3096 (0.0068) & 0.1803 (0.0055) & 0.1779 (0.0053) \\
Density: NLL selection & 5.4297 (0.0642) & 0.2889 (0.0065) & 0.1710 (0.0053) & 0.1695 (0.0052) \\
CTI: fixed QRF & 1.1620 (0.0259) & 0.1088 (0.0034) & 0.0707 (0.0025) & 0.0762 (0.0027) \\
HPD score & 5.9223 (0.0682) & 0.3237 (0.0065) & 0.1863 (0.0056) & 0.1830 (0.0055) \\
DCP-opt score & 7.4235 (0.0924) & 0.5075 (0.0097) & 0.2918 (0.0090) & 0.2710 (0.0079) \\
Gao-style, $k=2$ & 5.5337 (0.0653) & 0.3354 (0.0071) & 0.1923 (0.0056) & 0.1876 (0.0054) \\
Ad-EffOrt: tuned & 3.4779 (0.0838) & 0.9653 (0.0151) & 0.6491 (0.0206) & 0.7011 (0.0202) \\
\addlinespace
\multicolumn{5}{@{}l}{\emph{B. Serial splits; one numerical thread}}\\
{Density: length selection} & {4.8080 (0.0370)} & {0.2899 (0.0026)} & {0.1553 (0.0019)} & {0.1554 (0.0030)} \\
{WQ: fixed$^\dagger$} & {3.0437 (0.0205)} & {0.3351 (0.0055)} & {0.1410 (0.0017)} & {0.1086 (0.0015)} \\
{WQ: selected} & {5.7629 (0.0332)} & {0.5583 (0.0069)} & {0.2373 (0.0026)} & {0.1963 (0.0025)} \\
\bottomrule
\end{tabular}
\caption{{Seconds per split at shift ratio $3$ and
$p=0.9$: means (standard errors) over the 30 splits of
Table~\ref{tab:real-length}. Compare methods within each timing panel.
Costs follow the protocol above. The density selector is a common
reference; $^\dagger$fixed WQ includes unused tuning-input preparation.}}
\label{tab:runtime}
\end{table}

\paragraph{Cost comparisons.}
{
In Panel A, fixed CTI and NLL are faster than length selection on
all datasets. HPD costs 2.4\%--4.5\% more and gives shorter sets;
DCP-opt costs 1.28--1.64 times as much, including interval centers.
Excluding fitting/query setup, length selection costs 7.6--17.1 times
as much as NLL. Gao and tuned Ad-EffOrt have no consistent cost ordering.
In Panel B, selected WQ is slower on all four datasets; fixed WQ is
faster on Power Plant, Real Estate, and Yacht. These comparisons depend
on the implementations and unequal budgets in Section~\ref{sec:real-data}.
}

{
\section{Related work and benchmark scope}\label{sec:related}

\paragraph{Finite-sample optimality.}
\citet{faulkenberry1973prediction} constructs prediction regions by
conditioning on a sufficient statistic. \citet{hoff2023} combines
conditional testing with Bayes risk minimization under frequentist
coverage constraints; \citet[Theorem~8 and Corollary~2]{koning2025fuzzy}
derive conformal optimality within an e-value framework.
\citet*{tibshirani2026testing} give a direct testing derivation and explain
these precedents. Our proof applies conditional optimization to weighted
permutation constraints, with risk under $P^m\times Q$ and validity over
$\mathcal H_w$.

\paragraph{Density estimation and prediction sets.}
Weighted calibration and density-based regions are established
\citep{tibshirani2019shift,izbicki2020density,izbicki2022hpd}.
\citet{wang2025density} estimate conditional densities by
reference-distribution classification for weighted treatment-effect
inference, motivated by a population density-threshold optimum.
Our finite-$m$ benchmark includes calibration randomness.
CTI supplies our equal-mass quantile-bin ordering \citep{luo2025cti}.
Full CD-split partitions inputs and calibrates locally; our global
comparator uses one calibration cell. HPD-split and DCP transform
estimated densities or CDFs, with DCP also adjusting interval shape
\citep{izbicki2022hpd,chernozhukov2021dcp}. Their benchmarks and
\citet{gao2025volume}'s restriction to $k$ intervals differ from ours.
Other interval methods include CQR, CHR, and CIR/CIR+
\citep{romano2019cqr,sesia2021histograms,guo2026interquantile}.

\paragraph{Selection and data reuse.}
CPL optimizes length under coverage moments \citep{kiyani2024length};
EffOrt and Ad-EffOrt optimize error quantiles \citep{lebars2025volume}.
VFCP selects by width before independent final calibration
\citep{yang2025selection}. ModSel-CP and ModSel-CP-LOO incorporate the
hypothesized response into selection, retaining exchangeable finite-sample
coverage without a separate selection/calibration split and providing
width guarantees under candidate-class regularity
\citep{liang2026selection}. Our separate-pool selector uses fresh weighted
calibration and targets a fixed-$m$ oracle. Their published guarantees
assume exchangeability, so we do not include a weighted adaptation.

ORCA learns radii around conditional samples \citep{zheng2026orca};
WGCP combines generated quantiles with estimated shift weights
\citep{li2026generative}. \citet{isaev2026cqr} study nonasymptotic CQR and
fixed-score calibration limits under known shift. Theorem~\ref{wo:oracle}
compares all rules valid over $\mathcal H_w$ at the same calibration size.

}

\section{Discussion}\label{sec:discussion}

The consistency results require growing estimation resolution and
selection resources; the original fixed-library, 30-resample experiments
do not implement these sequences. Optimality requires validity over
$\mathcal H_w$ (Appendix~\ref{wo:limitations}) and concerns marginal
coverage and volume on the stated response scale. It does not imply
conditional coverage, optimality among $k$-interval sets, or shortest
original-unit length after a nonlinear transformation. All formal
coverage guarantees use the true density ratio; estimated-weight error
remains outside this analysis.

\section*{Acknowledgments and disclosures}
\textbf{Funding.} Supported in part by the Simons Foundation (\mbox{SFI-MPS-TSM-00025603}).

\textbf{Competing interests.} The author declares no competing interests.

\textbf{AI use disclosure.}
GPT-6 contributed to this work at the level of a research collaborator.
The author takes full responsibility for the content of this paper.

\FloatBarrier
\appendix
\section{The validity class and infinite-volume risks}
\label{wo:limitations}

\paragraph{Validity only for iid laws permits a smaller risk.}
The larger null class in Theorem~\ref{wo:oracle} is necessary for its
claim. A counterexample already exists when $w\equiv1$.
Take $m=1$, $p=3/4$, $\mathcal Y=[0,1]$ with Lebesgue measure, and
$A=[0,1/4]$. Report $A$ when the calibration response belongs to $A$,
and report $[0,1]$ otherwise. For any iid response law, writing
$\theta=P(Y\in A)$ gives coverage
\[
 \theta+(1-\theta)^2=1-\theta+\theta^2\ge3/4.
\]
For the reference density equal to $2$ on $A$ and $2/3$ on its
complement, $\theta=1/2$ and this rule has expected length $5/8$.
The inverse-density oracle has inclusion probabilities $7/8$ on $A$
and $5/8$ on its complement, hence expected length
\[
 (1/4)(7/8)+(3/4)(5/8)=11/16>5/8.
\]
The iid-valid rule fails on a bag with one response in each part:
its orbitwise coverage is $1/2$. Thus optimality among methods valid
over all weighted-exchangeable laws cannot be replaced by optimality
among methods required to be valid only over iid-source shift laws.

\begin{proposition}[An unavoidable infinite-volume event]
\label{wo:infinite-risk}
Suppose $\mu(\mathcal Y)=\infty$. If
\begin{equation}\label{wo:heavy-test}
 \mathbb P_{T\sim P^m,\,X\sim Q_X}
 \left\{p\,w(X)>\alpha\sum_{i=1}^m w(X_i)\right\}>0,
\end{equation}
every rule valid over $\mathcal H_w$ has $R_m(C)=\infty$.
In particular, this holds at every finite $m$ when $w(X)$ is
essentially unbounded under $Q_X$.
\end{proposition}
\begin{proof}
First symmetrize the rule as in Proposition~\ref{wo:equivalence}, which
preserves its target risk. At a fixed $T,x$ with $w(x)>0$, apply
\eqref{wo:bag-constraint} to each augmented candidate $(x,y)$.
The total contribution from the calibration positions is at most
$W_T=\sum_iw(X_i)$. Therefore, for every $y$,
\[
 \mathbb P_U\{y\in C(T,x,U)\}
 \ge\left[p-\frac{\alpha W_T}{w(x)}\right]_+.
\]
On the event in \eqref{wo:heavy-test}, this is a positive constant
over a set of infinite $\mu$-measure. Tonelli's theorem gives infinite
conditional expected volume and then infinite risk. If $w(X)$ is
essentially unbounded under $Q_X$, the conditional probability of
\eqref{wo:heavy-test} is positive for every finite value of $W_T$.
\end{proof}

The oracle theorem therefore extends to infinite reference measure,
but its value need not be finite. Quantitative claims about excess
length require a finite oracle risk or a finite-volume evaluation scale.
A fixed transformation of an unbounded response into a bounded interval
changes the volume objective: transforming prediction sets back requires
the inverse-transformation Jacobian when measuring original-scale length.

\section{Synthetic distributions and complete configuration grid}
\label{app:synthetic}

{
Section~\ref{sec:synthetic} specifies the density pairs, sampling
design, bin counts, and bandwidths.
Quantile histograms use empirical $j/B$ quantiles without grid
shrinkage and fixed outer endpoints 0 and 1.
Each Gaussian kernel is normalized on $[0,1]$ before bin integration;
the resulting kernel histogram is mixed with the uniform density
using weight $10^{-6}$.
}
Fixed CTI uses $B=20$; Gao-style DP selects among the four quantile
fits. Ad-EffOrt is not run in this two-group comparison.

{
The accompanying \path{Supplementary_results/} files report
sensitivity to calibration size and shift at $n=2000$, together with
all 24 original configurations: 50 repetitions and 11 fitted methods
give 13,200 rows, with oracle comparisons and paired standard errors.
} {
\paragraph{Oracle evaluation.}
For each score threshold $t>0$ arising from weighted calibration of
$s^*(x,y)=1/f(y\mid x)$, we compute the intervals forming
$\{y\in[0,1]:f(y\mid x)\ge1/t\}$.
Interior endpoints solve $f(y\mid x)=1/t$ on each monotone-density
branch, with absolute root tolerance $10^{-13}$; outer endpoints may
be 0 or 1. At final evaluation, an interval $[a,b]$ has conditional
probability $F_g(b)-F_g(a)$ given $X=g$. These CDFs use the uniform
and beta components of the specified mixtures.
For each reported mean, the standard error is the sample standard
deviation divided by $\sqrt{50}$. Oracle-gap errors are calculated
from the 50 within-repetition differences, rather than from separate
method and oracle standard errors.
}

\section{Real-data protocol and implementation details}
\subsection{Real-data protocol and implemented comparators}
\label{sec:real-protocol}

\paragraph{Data and splits.}
Each observation is $(X_i,Y_i)$, with feature vector
$X_i=(X_{i1},\ldots,X_{id})\in\mathbb R^d$ and scalar response $Y_i$.
There are $d$ features; $X_{ij}$ is feature $j$ of observation $i$.
For $x=(x_1,\ldots,x_d)$, $x_j$ is its $j$th coordinate.

We retain all numeric observations in four UCI datasets, remove Real Estate's
identifier column, and use Power Plant's first worksheet once.
The covariate shifts are designed.\footnote{Primary UCI
data references: Power Plant, \url{https://doi.org/10.24432/C5002N};
Concrete, \url{https://doi.org/10.24432/C5PK67}; Real Estate,
\url{https://doi.org/10.24432/C5J30W}; Yacht,
\url{https://doi.org/10.24432/C5XG7R}. Retrieved October 19, 2026.}
\begin{center}
\begin{tabular}{lrrl}
\hline
Dataset & Rows & Features & Shift variable\\
\hline
Power Plant & 9568 & 4 & Ambient temperature\\
Concrete & 1030 & 8 & Age\\
Real Estate & 414 & 6 & Distance to nearest MRT station\\
Yacht & 308 & 6 & Froude number\\
\hline
\end{tabular}
\end{center}
{
Methods share each split, using the allocation in Section~\ref{sec:real-data}
with the first four counts rounded down.
} Each dataset uses 30 repetitions (identifiers 10--39), seed 20261020,
and dataset-specific random streams; pilots are excluded.

\paragraph{Response scale and shift.}
Using the training responses $Y_1,\ldots,Y_n$, where
$n=N_{\rm train}$, define
\[
\begin{aligned}
 c
 &= \widehat q_{0.5},\\
 \mathrm{IQR}_{\rm train}
 &= \widehat q_{0.75}-\widehat q_{0.25},\\
 \mathrm{SD}_{\rm train}
 &= \left\{\frac1n\sum_{i=1}^n
       (Y_i-\overline Y_{\rm train})^2\right\}^{1/2},
 \qquad
 \overline Y_{\rm train}=\frac1n\sum_{i=1}^n Y_i,
\end{aligned}
\]
where $\widehat q_\tau$ is the empirical $\tau$-quantile of training
responses, linearly interpolated between order statistics. Set
\[
 a=\max\{\mathrm{IQR}_{\rm train},
          0.1\,\mathrm{SD}_{\rm train},10^{-12}\}.
\]
We apply the sigmoid transformation
{
\begin{equation}\label{eq:response-sigmoid}
 \widetilde Y_i
 =\sigma\!\left(\frac{Y_i-c}{a}\right),
 \qquad
 \sigma(t)=(1+e^{-t})^{-1}.
\end{equation}
}
Within each repetition, apply the same training-fitted $c$ and $a$
to the remaining data partitions. Set $\mathcal Y=[0,1]$; reported lengths are Lebesgue
measures on this transformed scale, not in original response units.
For the designated input feature $j$, define
\[
 g(x)=\mathbf 1\!\left\{
 x_j\ge \operatorname{median}(X_{1j},\ldots,X_{nj})
 \right\},
 \qquad
 v_r(x)=1+(r-1)g(x),\quad {r\in\{1,3,10,20\}}.
\]
For each $r$, we define the target distribution $Q$ by
\[
 dQ(x,y)
 =\frac{v_r(x)}{\mathbb E_P[v_r(X)]}\,dP(x,y).
\]
{Thus $r=1$ gives $Q=P$, while $r>1$ assigns relative weights
$1$ and $r$ to the lower and upper groups, respectively.}

\paragraph{Fitting the regression forests.}
Fit an ensemble of regression trees \citep{meinshausen2006} by the
squared-error bootstrap procedure below, producing the leaf partitions
in Section~\ref{sec:forest-density}.
{
Use $N_{\rm tree}$ trees and responses $\widetilde Y_i$;
}
training and fitting randomization are independent of final calibration
and test data. Numerical settings follow the construction.

\begin{enumerate}
\item \emph{Draw a training resample for each tree.}
For tree $t$, draw $n$ indices independently and uniformly with replacement
from $\{1,\ldots,n\}$. Let $b_{ti}$ count the draws of observation $i$,
so $\sum_i b_{ti}=n$; counts may be zero or exceed one.

\item \emph{Learn the tree's input-space partition.}
In tree $t$, node $\nu$ has input region $A_{t,\nu}\subseteq\mathcal X$.
The root is $A_{t,\mathrm{root}}=\mathcal X$, containing
$X_1,\ldots,X_n$. Observation $i$ has fitting weight $b_{ti}$;
those with $b_{ti}=0$ do not affect fitting.

At each node $\nu$, draw a fresh initial candidate feature set
$\mathcal J^{(0)}_{t,\nu}$ uniformly from all subsets of
$\{1,\ldots,d\}$ of size
\[
 k_{\rm feat}=\max\{1,\lfloor0.8d\rfloor\}.
\]
Features constant among bootstrap observations in $A_{t,\nu}$ cannot
split the node. The implementation may then examine additional
features; $\mathcal J_{t,\nu}$ is the set actually examined.

For $j\in\mathcal J_{t,\nu}$, thresholds $u$ are midpoints between
consecutive distinct feature-$j$ values among bootstrap observations
in $A_{t,\nu}$. Split $(j,u)$ gives
\[
\begin{aligned}
 A_{t,\nu}^{\mathrm{left}}(j,u)
 &=A_{t,\nu}\cap\{x\in\mathcal X:x_j\le u\},\\
 A_{t,\nu}^{\mathrm{right}}(j,u)
 &=A_{t,\nu}\cap\{x\in\mathcal X:x_j>u\}.
\end{aligned}
\]
The children partition $A_{t,\nu}$.

For any node region $B$ with positive bootstrap weight, define
\[
 \overline{\widetilde Y}_t(B)=
 \frac{\sum_{i:X_i\in B}b_{ti}\widetilde Y_i}
      {\sum_{i:X_i\in B}b_{ti}},
 \qquad
 S_t(B)=\sum_{i:X_i\in B}b_{ti}
 \bigl(\widetilde Y_i-\overline{\widetilde Y}_t(B)\bigr)^2.
\]
Let $\mathcal C_{t,\nu}$ contain the examined pairs $(j,u)$,
$j\in\mathcal J_{t,\nu}$, satisfying the minimum-leaf and other
splitting constraints. If nonempty, choose
\[
 (j^*,u^*)\in
 \operatorname*{arg\,min}_{(j,u)\in\mathcal C_{t,\nu}}
 \left\{
 S_t\bigl(A_{t,\nu}^{\mathrm{left}}(j,u)\bigr)
 +S_t\bigl(A_{t,\nu}^{\mathrm{right}}(j,u)\bigr)
 \right\}.
\]
This maximizes the reduction in squared error because $S_t(A_{t,\nu})$
is fixed.
{
Create two child nodes with regions
$A_{t,\nu}^{\mathrm{left}}(j^*,u^*)$ and
$A_{t,\nu}^{\mathrm{right}}(j^*,u^*)$.
}

\item \emph{Stop growing and retain the terminal leaves.}
Fix a minimum leaf size $\ell$ before fitting.
For node $\nu$ of tree $t$, define its sample count by
\[
 n_{t,\nu}
 =\#\{i:b_{ti}>0,\ X_i\in A_{t,\nu}\}.
\]
This counts distinct bootstrap-selected observations in the node.
Split recursively from the root; retain a terminal leaf when fitting
responses are equal or no {admissible} split gives two children with
at least $\ell$ observations each.

Each tree $t$ has terminal regions partitioning $\mathcal X$.
For query $x$, $\mathcal L_t(x)$ is the leaf containing $x$ in tree $t$.
Repeat for $t=1,\ldots,N_{\rm tree}$.

\item \emph{Construct the conditional-distribution weights.}
In each fitted tree $t$, route every original input $X_i$, including
those with $b_{ti}=0$, through the fitted feature--threshold rules.
Form $\mathcal I_t(x)$ as in Section~\ref{sec:forest-density}; assign
weight $1/|\mathcal I_t(x)|$ to each observation in this leaf and zero
outside, then average over trees as in \eqref{eq:forest-training-weights}.
Bootstrap multiplicities $b_{ti}$ affect fitting only, not this
original-sample averaging. These weights define the conditional CDF
and histograms.
\end{enumerate}
Forest-fitting resamples differ from the later selection resamples.
The local weights $a_i(x)$ also differ from the covariate-shift weights
$w(x)$ used for calibration and target evaluation.

\paragraph{Fitted library.}
Two 80-tree bootstrap regression forests \citep{meinshausen2006} use minimum leaf sizes 5 and 20,
and feature fraction 0.8. {Their original-sample observation
weights are defined in \eqref{eq:forest-training-weights}.} Six quantile
histograms use $B\in\{10,20,40\}$, with endpoints 0 and 1 and interior
quantiles shrunk 2\% toward the uniform grid. Six kernel histograms share
the forest weights, mix Gaussian kernels individually truncated to $[0,1]$,
use bandwidths $0.02,0.05,0.1$, and integrate onto 80 equal-width cells.
Their masses receive a $10^{-6}$ uniform mixture.

\paragraph{Selection and calibration.}
{
Length selection uses \eqref{eq:pool-selection} with $r=30$ and the
allocations in Section~\ref{sec:real-data}. Methods share resamples and
input pools. The $30$ resamples reuse one labeled pool, not $30$
independent samples from $P$. Final calibration uses \eqref{eq:calibrated}
at level $p$ ($p=0.9$ in the main benchmark).
}

\paragraph{Implemented methods.}
Write $\Delta_j,\pi_j$ for fitted cell width and mass, and
$\widehat f_j=\pi_j/\Delta_j$. Smaller scores are preferred;
selection minimizes bootstrap length unless specified otherwise.
Result-file identifiers are:
\begin{enumerate}
\item \texttt{density\_length}: select among all 12 histograms using
      score $\Delta_j/\pi_j$.
\item \texttt{density\_singlebag}: select the same 12 scores using the
      original labeled pool without resampling.
\item \texttt{density\_nll}: select the same 12 fits by target-weighted
      negative log density, then use their inverse-density score.
\item \texttt{cti\_fixed}: fixed leaf-size-20, $B=20$ quantile histogram;
      its score is order-equivalent to CTI bin width \citep{luo2025cti}.
\item \texttt{quantile\_length}: density selection over the six quantile fits.
\item \texttt{kernel\_length}: density selection over the six kernel fits.
\item \texttt{hpd\_length}: select among 12 fits using
      $\sum_{\ell:\widehat f_\ell>\widehat f_j}\pi_\ell$
      \citep{izbicki2022hpd}; the strict inequality preserves density ties.
\item \texttt{pit\_length}: 12 fits, score
      $|\widehat F_x(z)-1/2|$ \citep{chernozhukov2021dcp}.
\item \texttt{dcp\_opt\_length}: 12 fits, score
      $|\widehat F_x(z)-b_x-p/2|$, where $b_x$ minimizes
      $\widehat Q_x(b+p)-\widehat Q_x(b)$ over $[0,1-p]$
      \citep{chernozhukov2021dcp}.
\item \texttt{cqr\_length}: 12 fits, score
      $\max\{\widehat Q_x(0.05)-z,z-\widehat Q_x(0.95)\}$
      \citep{romano2019cqr}.
{
\item \texttt{wq\_rf\_fixed}: reference-classification density score
      with minimum leaf size 5 \citep{wang2025density}.
\item \texttt{wq\_rf\_selected}: select minimum leaf size from
      $\{5,20\}$ with the same length objective.
}
\item \texttt{gao\_union2\_length}: six quantile fits; exact whole-cell DP
      selects $\lceil pB\rceil$ cells in at most two intervals.
      Adjacent-cell expansion and boundary-cell contraction create a nested
      family, scored by entry rank divided by $B$ \citep{gao2025volume}.
\item \texttt{adeffort}: fixed smoothing 0.05, score
      $|z-\widehat c(x)|-\widehat a(x)$ \citep{lebars2025volume}.
\item \texttt{adeffort\_tuned}: select smoothing from $\{0.02,0.05,0.1\}$
      with the same length objective.
\end{enumerate}

Ad-EffOrt uses a 32-unit tanh center, 150 MSE warm-start steps, 500 kernel
quantile-gradient Adadelta steps, training-only checkpoint selection, and
a 120-iteration boosted residual-quantile model. Gao uses nominal anchor
mass $p$, without the theoretical $n^{-1}+3\delta$ inflation. Both are
independent weighted adaptations, not reproductions of published models
or results. Candidate counts and computational budgets differ by family.

{
\paragraph{Reference-classification scores.}\label{app:wq-protocol}
We use the density-estimation component of Algorithm~1 in
\citet{wang2025density}, omitting its treatment-effect stage.
For each training input, draw a reference response from an untruncated
Gaussian with the training-response mean and $1.5$ times its empirical
standard deviation on the transformed scale. A binary forest distinguishes
observed from reference pairs. With fitted observed-pair probability
$\widehat\pi(x,y)$ and reference density $f_0$, use score
$-f_0(y)\widehat\pi(x,y)/(1-\widehat\pi(x,y))$.
Clip probabilities to $[10^{-6},1-10^{-6}]$, without input-specific
renormalization. Each classifier has 80 bootstrap trees and feature
fraction $0.8$; both leaf-size candidates share reference draws.

{
Integrate lengths
} on classifier response partitions using Gaussian level equations,
without a response grid. The estimate need not integrate to one, so
Proposition~\ref{prop:fixed-m-density} does not apply; weighted
calibration still gives coverage.
}

\paragraph{Evaluation.}
Use weighted source-test averages for target coverage and mean length,
integrating randomization analytically for coverage, length, and
positive-length component counts. Report shift-group coverage and timing.
Standard errors describe split variation within each fixed dataset;
method differences are paired by split.

\paragraph{Observed optimization and split diagnostics.}
All 90 Power Plant Ad-EffOrt fits (30 splits, three smoothing values)
retain an MSE warm-start checkpoint: later steps do not improve the
training quantile criterion. This concerns our implementation, not an
optimally tuned original pipeline. All final splits are retained,
including Yacht coverage 0.6681 for the density selector in repetition
29 and 0.5969 for HPD in repetition 32. Marginal coverage averages
over calibration samples, not each realized split.

{
\subsection{Protocol for the controlled score and selection diagnostic}
\label{app:diagnostic-protocol}

\paragraph{Real data and comparison units.}
{
Use Appendix~\ref{sec:real-protocol}'s protocol at ratio $3$ and
$p=0.9$. Evaluate inverse-density, HPD, and DCP-opt on all 12 fits
with shared final calibration and test data excluded from selection.
NLL15 and length selection share the $15\%$ labeled pool; length
selection also uses the $10\%$ input pool without responses. NLL25
adds those responses for a budget check.
{
Length selection uses $r=30$ bootstrap samples.
}

Anchors are the leaf-size-20 kernel with bandwidth $0.05$, NLL15,
and each score's length-selected fit. Retain all 12-fit comparisons;
average within splits before computing errors. Descriptive paired
$95\%$ intervals use $t_{29}$ without multiplicity adjustment and
measure split variation in reused datasets. This follow-up is exploratory.
For Yacht, change inverse density to HPD at the inverse-density-selected
fit, then switch to the HPD-selected fit. These score and fit components
sum to the total difference in each split. 
}

\paragraph{Computational comparisons.}
{
Timing is serial and single-threaded; each method pays the common
12-fit training/query cost under Section~\ref{sec:runtime}'s convention.
}
\paragraph{Known-density check.}
The six scenarios use the four density pairs in
Section~\ref{sec:synthetic}, plus the asymmetric and separated-mode
pairs after removing their uniform components and renormalizing.
Source group probabilities are $(1/2,1/2)$ and target probabilities
are $(1/4,3/4)$. We fix $m=200$ and $p=0.9$, use $50$ repetitions
(indices 10--59, seed 20261026), and take nested training samples
of $500$, $2000$, and $8000$ observations per group.
All scores and training sizes share 200 labeled source observations,
200 independent source inputs, and fresh final calibration in each
repetition. Length selection uses 30 bootstrap samples; NLL uses the
same 200 labels.
{
The seven fitted candidates are those specified in
Appendix~\ref{app:synthetic}.
}

True scores use analytic beta-mixture densities and CDFs.
Density contours and HPD inversions use scalar roots on monotone-density
branches, with root tolerance $10^{-12}$. The true DCP-opt center is
computed by comparing quantile widths at boundary points and stationary
candidates found from derivative sign changes on a 513-point bracket grid;
resolution sensitivity was checked separately.
Integrate conditional interval masses and common randomization for
length and coverage. The true reference is training-size independent
and counted once per repetition, not three times.
}

{
\begin{table}[!htbp]
\centering\small

\setlength{\tabcolsep}{7pt}
\begin{tabular}{@{}lrr@{}}
\toprule
Scenario & True HPD $-$ true density & True DCP-opt $-$ true density \\
\midrule
Asymmetric & $+0.00055$ (0.00068) & $+0.00338$ (0.00107) \\
Heteroscedastic & $+0.00033$ (0.00125) & $+0.00033$ (0.00125) \\
Separated modes & $+0.00022$ (0.00082) & $+0.20079$ (0.00795) \\
Unequal modes & $-0.00063$ (0.00118) & $+0.20772$ (0.00682) \\
Asymmetric, no floor & $-0.00024$ (0.00102) & $+0.00123$ (0.00168) \\
Separated modes, no floor & $+0.00149$ (0.00102) & $+0.25701$ (0.00404) \\
\bottomrule
\end{tabular}
\caption{True-score length differences in the diagnostic simulations,
with $m=200$. Entries are paired means (standard errors) over 50
repetitions. True references are counted once per repetition;
uncorrected descriptive $95\%$ intervals use the $t_{49}$ quantile.}
\label{tab:diagnostic-true-scores}
\end{table}
}

\section{Additional real-data results}

{
\subsection{Inverse-density and HPD scores at lower coverage levels}
\label{app:hpd-low-coverage}

This exploratory comparison follows Appendix~\ref{sec:real-protocol},
using $p\in\{0.3,0.4,0.5\}$, relative weights $(1,10)$ and $(1,20)$,
and the same 12 fitted densities for both methods.
Table~\ref{tab:hpd-low-coverage-r10} reports transformed-scale length,
coverage, and runtime; timing follows Section~\ref{sec:runtime}.

Inverse density has shorter mean sets in nine of 12 combinations
at ratio $10$ and ten at ratio $20$. Unadjusted $95\%$ paired
length intervals favor inverse density in seven combinations at
each ratio; the remaining intervals contain zero. Inverse density
is faster in all 24 combinations; HPD takes $1.02$--$1.05$ times
as long. Split-level results and paired intervals accompany this table.

Achieved coverage also matters. On Real Estate at ratio $20$
and $p=0.3$, inverse density and HPD attain $27.35\%$ and
$30.27\%$. On Yacht at ratio $20$ and $p=0.5$, HPD is shorter
but attains $46.64\%$, compared with $49.05\%$ for inverse density.
These length differences therefore cannot be read as gains at
equal achieved coverage.

\begin{table}[!htbp]
\centering\small
\setlength{\tabcolsep}{3pt}
\begin{tabular}{@{}lcrrrrrr@{}}
\toprule
& & \multicolumn{2}{c}{Length (SE)} & \multicolumn{2}{c}{Coverage (\%)} & \multicolumn{2}{c}{Time (s)} \\
\cmidrule(lr){3-4}\cmidrule(lr){5-6}\cmidrule(l){7-8}
Dataset & $p$ & Density & HPD & Density & HPD & Density & HPD \\
\midrule
\multicolumn{8}{@{}l}{\emph{Relative shift weights $(1,10)$}}\\
Power Plant & 0.3 & 0.01746 (0.00021) & 0.02004 (0.00019) & 30.27 & 30.83 & 5.432 & 5.567 \\
 & 0.4 & 0.02486 (0.00023) & 0.02635 (0.00026) & 40.60 & 40.87 & 5.438 & 5.558 \\
 & 0.5 & 0.03347 (0.00033) & 0.03445 (0.00024) & 50.53 & 50.21 & 5.436 & 5.571 \\
\addlinespace[2pt]
Concrete & 0.3 & 0.04082 (0.00222) & 0.04778 (0.00184) & 30.87 & 31.48 & 0.331 & 0.348 \\
 & 0.4 & 0.05858 (0.00239) & 0.06597 (0.00252) & 40.21 & 40.35 & 0.331 & 0.346 \\
 & 0.5 & 0.08078 (0.00255) & 0.08629 (0.00277) & 50.23 & 49.80 & 0.332 & 0.347 \\
\addlinespace[2pt]
Real Estate & 0.3 & 0.04149 (0.00285) & 0.05187 (0.00292) & 27.67 & 29.16 & 0.175 & 0.181 \\
 & 0.4 & 0.05925 (0.00309) & 0.06682 (0.00333) & 38.51 & 41.59 & 0.175 & 0.182 \\
 & 0.5 & 0.07590 (0.00334) & 0.08545 (0.00358) & 49.05 & 50.49 & 0.175 & 0.181 \\
\addlinespace[2pt]
Yacht & 0.3 & 0.01512 (0.00155) & 0.01487 (0.00159) & 29.10 & 28.12 & 0.170 & 0.175 \\
 & 0.4 & 0.02132 (0.00143) & 0.02065 (0.00145) & 39.41 & 37.60 & 0.170 & 0.176 \\
 & 0.5 & 0.02943 (0.00172) & 0.02644 (0.00153) & 47.85 & 46.85 & 0.171 & 0.176 \\
\addlinespace
\multicolumn{8}{@{}l}{\emph{Relative shift weights $(1,20)$}}\\
Power Plant & 0.3 & 0.01754 (0.00022) & 0.02023 (0.00019) & 30.21 & 30.81 & 5.428 & 5.562 \\
 & 0.4 & 0.02509 (0.00027) & 0.02651 (0.00027) & 40.66 & 40.83 & 5.432 & 5.560 \\
 & 0.5 & 0.03361 (0.00034) & 0.03467 (0.00026) & 50.53 & 50.07 & 5.434 & 5.562 \\
\addlinespace[2pt]
Concrete & 0.3 & 0.04076 (0.00228) & 0.04831 (0.00194) & 30.88 & 31.61 & 0.331 & 0.346 \\
 & 0.4 & 0.05896 (0.00245) & 0.06663 (0.00250) & 40.32 & 40.61 & 0.332 & 0.348 \\
 & 0.5 & 0.08090 (0.00257) & 0.08639 (0.00282) & 50.23 & 49.83 & 0.331 & 0.346 \\
\addlinespace[2pt]
Real Estate & 0.3 & 0.04062 (0.00275) & 0.05223 (0.00284) & 27.35 & 30.27 & 0.175 & 0.182 \\
 & 0.4 & 0.05864 (0.00303) & 0.06642 (0.00342) & 38.73 & 41.38 & 0.176 & 0.182 \\
 & 0.5 & 0.07681 (0.00312) & 0.08542 (0.00355) & 50.02 & 50.53 & 0.175 & 0.181 \\
\addlinespace[2pt]
Yacht & 0.3 & 0.01600 (0.00167) & 0.01657 (0.00173) & 28.88 & 29.14 & 0.170 & 0.176 \\
 & 0.4 & 0.02302 (0.00154) & 0.02197 (0.00150) & 39.98 & 37.53 & 0.171 & 0.176 \\
 & 0.5 & 0.03057 (0.00168) & 0.02752 (0.00161) & 49.05 & 46.64 & 0.170 & 0.176 \\
\bottomrule
\end{tabular}
\caption{Inverse-density and HPD selection at lower coverage levels:
means over 30 shared splits, with length standard errors in parentheses.
Both methods use the same 12 densities and training-transformed $[0,1]$
responses. Timing follows Section~\ref{sec:runtime} and charges the full
shared fitting cost to each method. Coverage and timing standard errors,
and paired intervals, are in the reproduction files.}
\label{tab:hpd-low-coverage-r10}\label{tab:hpd-low-coverage-r20}
\end{table}

}

\subsection{No-shift and group-coverage controls}
\label{app:real-controls}

Tables~\ref{tab:no-shift-length} and~\ref{tab:real-groups} report
the no-shift control and shift-group coverage for the main benchmark.

\begin{table}[!htbp]
\centering\small
\setlength{\tabcolsep}{4pt}
\renewcommand{\arraystretch}{1.08}
\begin{tabular}{@{}lrrrr@{}}
\toprule
Method & Power Plant & Concrete & Real Estate & Yacht \\
\midrule
Density: length selection & \shortstack[r]{0.0895 (0.0008)\\0.8997 (0.0022)} & \shortstack[r]{0.2273 (0.0044)\\0.8982 (0.0054)} & \shortstack[r]{0.2600 (0.0064)\\0.8978 (0.0069)} & \shortstack[r]{0.0741 (0.0045)\\0.8986 (0.0132)} \\
\addlinespace[2pt]
Density: NLL selection & \shortstack[r]{0.0882 (0.0005)\\0.8975 (0.0021)} & \shortstack[r]{0.2320 (0.0035)\\0.9009 (0.0051)} & \shortstack[r]{0.2559 (0.0058)\\0.8967 (0.0066)} & \shortstack[r]{0.0713 (0.0032)\\0.8837 (0.0132)} \\
\addlinespace[2pt]
Density: one tuning bag & \shortstack[r]{0.0897 (0.0008)\\0.9003 (0.0022)} & \shortstack[r]{0.2277 (0.0038)\\0.9006 (0.0049)} & \shortstack[r]{0.2729 (0.0110)\\0.9051 (0.0065)} & \shortstack[r]{0.0747 (0.0044)\\0.8995 (0.0132)} \\
\addlinespace[2pt]
CTI: fixed QRF & \shortstack[r]{0.1014 (0.0005)\\0.8984 (0.0022)} & \shortstack[r]{0.3038 (0.0043)\\0.8933 (0.0060)} & \shortstack[r]{0.3189 (0.0069)\\0.8973 (0.0082)} & \shortstack[r]{0.1545 (0.0085)\\0.8980 (0.0125)} \\
\addlinespace[2pt]
Density: quantiles only & \shortstack[r]{0.0953 (0.0011)\\0.9015 (0.0024)} & \shortstack[r]{0.2381 (0.0036)\\0.8984 (0.0050)} & \shortstack[r]{0.2947 (0.0124)\\0.8921 (0.0089)} & \shortstack[r]{0.0740 (0.0048)\\0.8914 (0.0143)} \\
\addlinespace[2pt]
Density: kernels only & \shortstack[r]{0.0882 (0.0005)\\0.8975 (0.0021)} & \shortstack[r]{0.2253 (0.0037)\\0.8999 (0.0054)} & \shortstack[r]{0.2588 (0.0060)\\0.8982 (0.0067)} & \shortstack[r]{0.0708 (0.0030)\\0.8993 (0.0135)} \\
\addlinespace[2pt]
HPD score & \shortstack[r]{0.0878 (0.0005)\\0.8982 (0.0021)} & \shortstack[r]{0.2113 (0.0031)\\0.8970 (0.0052)} & \shortstack[r]{0.2526 (0.0053)\\0.8973 (0.0072)} & \shortstack[r]{0.0517 (0.0022)\\0.8930 (0.0118)} \\
\addlinespace[2pt]
PIT score & \shortstack[r]{0.0877 (0.0006)\\0.8985 (0.0024)} & \shortstack[r]{0.2110 (0.0031)\\0.8971 (0.0058)} & \shortstack[r]{0.2542 (0.0064)\\0.8973 (0.0076)} & \shortstack[r]{0.0658 (0.0047)\\0.8733 (0.0138)} \\
\addlinespace[2pt]
DCP-opt score & \shortstack[r]{0.0868 (0.0005)\\0.8974 (0.0020)} & \shortstack[r]{0.2069 (0.0032)\\0.8943 (0.0060)} & \shortstack[r]{0.2509 (0.0058)\\0.8929 (0.0076)} & \shortstack[r]{0.0620 (0.0035)\\0.8853 (0.0124)} \\
\addlinespace[2pt]
CQR score & \shortstack[r]{0.0895 (0.0005)\\0.8986 (0.0023)} & \shortstack[r]{0.2291 (0.0032)\\0.9027 (0.0064)} & \shortstack[r]{0.2579 (0.0048)\\0.8961 (0.0077)} & \shortstack[r]{0.1455 (0.0081)\\0.8849 (0.0099)} \\
\addlinespace[2pt]
{WQ: fixed} & \shortstack[r]{{0.1037 (0.0006)}\\{0.8976 (0.0021)}} & \shortstack[r]{{0.3060 (0.0065)}\\{0.8893 (0.0068)}} & \shortstack[r]{{0.3048 (0.0084)}\\{0.8900 (0.0099)}} & \shortstack[r]{{0.1602 (0.0080)}\\{0.8834 (0.0149)}} \\
\addlinespace[2pt]
{WQ: selected} & \shortstack[r]{{0.1037 (0.0006)}\\{0.8976 (0.0021)}} & \shortstack[r]{{0.3060 (0.0065)}\\{0.8893 (0.0068)}} & \shortstack[r]{{0.3048 (0.0084)}\\{0.8900 (0.0099)}} & \shortstack[r]{{0.1602 (0.0080)}\\{0.8834 (0.0149)}} \\
\addlinespace[2pt]
Gao-style, $k=2$ & \shortstack[r]{0.0902 (0.0006)\\0.9015 (0.0023)} & \shortstack[r]{0.2241 (0.0037)\\0.8969 (0.0051)} & \shortstack[r]{0.2994 (0.0141)\\0.8963 (0.0085)} & \shortstack[r]{0.0601 (0.0032)\\0.8942 (0.0104)} \\
\addlinespace[2pt]
Ad-EffOrt: fixed & \shortstack[r]{0.1011 (0.0004)\\0.8998 (0.0019)} & \shortstack[r]{0.2214 (0.0035)\\0.8990 (0.0047)} & \shortstack[r]{0.2898 (0.0058)\\0.8985 (0.0082)} & \shortstack[r]{0.1033 (0.0039)\\0.8945 (0.0106)} \\
\addlinespace[2pt]
Ad-EffOrt: tuned & \shortstack[r]{0.1011 (0.0004)\\0.8998 (0.0019)} & \shortstack[r]{0.2127 (0.0035)\\0.8964 (0.0056)} & \shortstack[r]{0.2836 (0.0061)\\0.9000 (0.0088)} & \shortstack[r]{0.0981 (0.0033)\\0.8971 (0.0109)} \\
\bottomrule
\end{tabular}
\caption{No-shift control at nominal coverage $0.9$. Each cell gives expected length on the first line and coverage on the second line, both as mean (split standard error) over 30 shared splits. Coverage is a probability; length uses the training-transformed $[0,1]$ scale. The splits and fitted estimators are shared with the shifted evaluation.}
\label{tab:no-shift-length}
\label{tab:no-shift-coverage}
\end{table}

\begin{table}[!htbp]
\centering\small
\setlength{\tabcolsep}{4pt}
\begin{tabular}{@{}lrrrr@{}}
\toprule
Method & Power Plant & Concrete & Real Estate & Yacht \\
\midrule
Density: length selection & \shortstack[r]{0.9053 (0.0024)\\0.8980 (0.0024)} & \shortstack[r]{0.8988 (0.0100)\\0.9006 (0.0056)} & \shortstack[r]{0.8403 (0.0104)\\0.9238 (0.0086)} & \shortstack[r]{1.0000 (0.0000)\\0.8708 (0.0187)} \\
\addlinespace[2pt]
HPD score & \shortstack[r]{0.9063 (0.0021)\\0.8945 (0.0024)} & \shortstack[r]{0.9163 (0.0090)\\0.8989 (0.0051)} & \shortstack[r]{0.8499 (0.0104)\\0.9163 (0.0097)} & \shortstack[r]{0.9918 (0.0036)\\0.8717 (0.0217)} \\
\addlinespace[2pt]
DCP-opt score & \shortstack[r]{0.9044 (0.0023)\\0.8958 (0.0024)} & \shortstack[r]{0.9152 (0.0087)\\0.8928 (0.0062)} & \shortstack[r]{0.8482 (0.0125)\\0.9128 (0.0093)} & \shortstack[r]{0.9942 (0.0024)\\0.8569 (0.0184)} \\
\addlinespace[2pt]
WQ: selected & \shortstack[r]{0.8919 (0.0029)\\0.9008 (0.0027)} & \shortstack[r]{0.8715 (0.0135)\\0.8936 (0.0067)} & \shortstack[r]{0.8481 (0.0158)\\0.9187 (0.0106)} & \shortstack[r]{0.9930 (0.0030)\\0.8420 (0.0254)} \\
\addlinespace[2pt]
Gao-style, $k=2$ & \shortstack[r]{0.9037 (0.0028)\\0.9005 (0.0026)} & \shortstack[r]{0.9052 (0.0115)\\0.8977 (0.0047)} & \shortstack[r]{0.8648 (0.0138)\\0.9093 (0.0100)} & \shortstack[r]{0.9844 (0.0047)\\0.8895 (0.0160)} \\
\addlinespace[2pt]
Ad-EffOrt: tuned & \shortstack[r]{0.9012 (0.0028)\\0.8979 (0.0021)} & \shortstack[r]{0.8872 (0.0092)\\0.9042 (0.0065)} & \shortstack[r]{0.8799 (0.0113)\\0.9135 (0.0100)} & \shortstack[r]{0.9715 (0.0077)\\0.8852 (0.0173)} \\
\bottomrule
\end{tabular}
\caption{Coverage in the two input groups defining the ratio-$3$ shift.
Each cell gives Group 0 on the first line and Group 1 on the second,
with split standard errors in parentheses. These are diagnostics,
not the coverage constraint optimized by the theorem.}
\label{tab:real-groups}
\end{table}

{
\subsection{Coverage levels and a joint score library}
\label{sec:coverage-joint}

We compare 12-score selectors at $p\in\{0.6,0.7,0.8,0.9\}$
and shift ratio 3 under
Appendix~\ref{sec:real-protocol}. Inverse density and HPD each use
12 densities; joint selection applies both scores to six, retaining
$B=20$ and $h\in\{0.02,0.1\}$ for each forest. This subset was fixed
before the follow-up, after examining the datasets and splits.
Selection and final calibration are repeated at every $p$.
Theorem~\ref{thm:pool-selection} applies under its assumptions.
Complete means, standard errors, and paired intervals are in
\mbox{\path{joint_score_comparison/Joint12_vs_HPD_summary.csv}}.
This separate serial, one-thread timing run follows
Section~\ref{sec:runtime}; compare times within this study.

Joint selection is faster in all 16 combinations: savings relative
to HPD are $35.1\%$--$35.3\%$ on Power Plant and $3.5\%$--$7.5\%$
elsewhere, from fewer histograms with both forests fully charged.
Its mean length is smaller than HPD in six combinations and inverse density in 13.
All six gains over HPD accompany lower coverage; on Real Estate at
$p=0.7$, coverages are $67.95\%$ and $70.06\%$, with paired $95\%$
difference interval $[-3.99,-0.21]$ percentage points.
Unadjusted $95\%$ paired length intervals favor joint selection only
on Power Plant at $p=0.7$ and HPD only on Concrete at $p=0.7$;
the other 14 contain zero. This does not establish consistent length
improvement or noninferiority to HPD.
}

\section{Implementation verification and reproducibility}
\label{app:verification}

{
Audits check candidate self-weight, weighted ties, common-$U$
randomization, histogram integration, CTI ordering, Gao DP, Ad-EffOrt
gradients, and agreement of direct and vectorized resample calculations.
Saved counts, tolerances, and independent summary/paired-difference checks
cover 3600 main real-data, 13200 original synthetic, 1800 vanishing-density,
and 1440 four-level records; pilots remain separate.

Main and WQ records are documented in \path{repro/README.txt} and
\path{wang_qiao_comparison/README.txt}. The vanishing-density
\path{experiment/README.txt} includes a separate 5000-calibration audit:
oracle coverages $0.90060$ (cluster SE $0.00052$) and $0.90011$ ($0.00054$),
beyond the 50 formal repetitions. Four-level code, results, and audits
are in \path{joint_score_comparison/}; its baselines exactly reproduce
preceding predictions.

{
For ratio-$10$, \path{dcp_shift10_comparison/README.txt} documents
960 records, candidate losses, stage timings, and independent checks
of weighted coverage, length integration, and paired summaries.
}

{
\path{Supplementary_results/} retains complete vanishing-density,
synthetic-configuration, and connected-component tables and archived
exploratory comparisons omitted here.
}

}
{
\section{Kernel histogram error bounds}\label{app:kde-proof}

\begin{proof}[Proof of Proposition~\ref{prop:kde-error}] 
All norms and integrals in this proof are on $[a,b]$ unless stated
otherwise. We first control density error within one group and
suppress its index. Write $\phi_h(v)=h^{-1}\phi(v/h)$ and
\[
 g_h(y)=\mathbb E_{\rm train}\widehat g_h(y)
       =\int_a^b\frac{\phi_h(y-t)}{Z_h(t)}f(t)\,dt,
 \qquad
 g_h^0(y)=\int_a^b\phi_h(y-t)f(t)\,dt.
\]
The Gaussian tail integral gives
\[
 \int_a^b(1-Z_h(t))\,dt
 \le 2h\int_0^\infty(1-\Phi(u))\,du
 =h\sqrt{2/\pi}.
\]
Using Lipschitz continuity and
$\int_{\mathbb R}|v|\phi_h(v)\,dv=h\sqrt{2/\pi}$,
\[
 \|g_h^0-f\|_1
 \le LKh\sqrt{2/\pi}
      +M_f\int_a^b(1-Z_h(y))\,dy
 \le(LK+M_f)h\sqrt{2/\pi}.
\]
Since $Z_h(t)\le1$, we have $g_h\ge g_h^0$ and
\[
 \|g_h-g_h^0\|_1
 =\int_a^b f(t)(1-Z_h(t))\,dt
 \le M_fh\sqrt{2/\pi}.
\]
Consequently,
$\|g_h-f\|_1\le(LK+2M_f)h\sqrt{2/\pi}$.
This bound includes the boundary normalization; it does not assume
a uniform pointwise bias bound at the endpoints.

Let $H_B$ replace a function by its average on each bin,
using the factor $B/L$.
It is an $L^1$ contraction. The binning error satisfies
\[
 \|H_Bf-f\|_1
 \le\frac{KB}{L}\sum_{j=1}^B\int_{J_j}\int_{J_j}|u-y|\,du\,dy
 =\frac{KL^2}{3B}.
\]
For $A_j(t)=\int_{J_j}k_h(y,t)\,dy$ and
$p_j=\mathbb E A_j(Y)$, we have $0\le A_j\le1$ and $\sum_jp_j=1$.
Independence within the group and Cauchy--Schwarz give
\[
\begin{aligned}
 \mathbb E_{\rm train}\|H_B\widehat g_h-H_Bg_h\|_1
 &=\sum_{j=1}^B
   \mathbb E\left|\frac1n\sum_{i=1}^n A_j(Y_i)-p_j\right|\\
 &\le\sum_{j=1}^B\sqrt{\frac{\operatorname{Var}(A_j(Y))}{n}}
 \le\sum_{j=1}^B\sqrt{\frac{p_j}{n}}
 \le\sqrt{\frac Bn}.
\end{aligned}
\]
Combining these inequalities proves the first bound in
\eqref{eq:kde-density-error} for each group and hence for their
$Q_X$-weighted average. The second follows from
$\|\widehat f-f\|_1
\le(1-\lambda)\|\widetilde f-f\|_1+\lambda\|1/L-f\|_1$
and $\|1/L-f\|_1\le2$.

{
{
For the length bound, \eqref{eq:uniform-mixture-invariance} and
Corollary~\ref{wo:density-ordering}, with $f\ge c$, give
}
\begin{equation}\label{eq:kde-l1-transfer}
 R_m(C_{s_{h,B,\lambda}})-R_m(C^*)
 \le\frac1c\,\mathbb E_{Q_X}
      \|\widetilde f_{h,B}(\cdot\mid X)-f(\cdot\mid X)\|_1.
\end{equation}
Taking expectations over the training data and using
\eqref{eq:kde-density-error} proves \eqref{eq:kde-length-error}.
Finally, positivity, finiteness, and independence of the fitted score
allow Proposition~\ref{wo:exact} to give exact coverage.
}
\end{proof}
}
{
\section{Forest CDF error and density error}
\label{app:forest-density-proof}

\begin{proof}[Proof of Proposition~\ref{prop:forest-density}] 
{
For any fixed $x$ and $y,t\in[a,b]$,
the Lipschitz condition gives
}
\[
 f(y\mid x)\le f(t\mid x)+K|y-t|.
\]
Since $\int_a^b f(t\mid x)\,dt=1$ and $L=b-a$,
averaging this inequality over $t\in[a,b]$ yields
\[
\begin{aligned}
 f(y\mid x)
 &\le \frac1L+\frac KL\int_a^b|y-t|\,dt\\
 &=\frac1L+\frac{K}{2L}
       \bigl\{(y-a)^2+(b-y)^2\bigr\}\\
 &\le\frac1L+\frac{KL}{2},
\end{aligned}
\]
where the last step uses
$(y-a)^2+(b-y)^2\le(b-a)^2=L^2$.
Taking the supremum over $x$ and $y$ gives
$M_f\le L^{-1}+KL/2$.
Thus $M_f$ is finite even when $\mathcal X$ is infinite.
Fix the training data and $x$ throughout the next two steps.

\paragraph{Kernel histograms.}
Define the kernel distribution function
\[
 H_h(v,t)=\int_a^v k_h(y,t)\,dy,
 \qquad a\le v,t\le b.
\]
For fixed $v$, this function is nonincreasing in $t$. Indeed,
if $W$ has density $k_h(\cdot,t)$, differentiation gives
\[
 \frac{\partial H_h(v,t)}{\partial t}
 =\frac1{h^2}
   \operatorname{Cov}\bigl(\mathbf 1\{W\le v\},W\bigr)\le0.
\]
Its total variation on $[a,b]$ is at most one. Integration by parts
against the signed measure with distribution function
$\widehat F_n(\cdot\mid x)-F_x$ therefore gives
\[
 \sup_{v\in[a,b]}
 \left|\int_a^b H_h(v,t)
       \,d\{\widehat F_n(t\mid x)-F_x(t)\}\right|
 \le\Delta_n(x).
\]
The endpoint terms vanish because the signed cumulative mass is zero
below $a$ and at $b$.
Let $g_h(y\mid x)=\int_a^b k_h(y,t)f(t\mid x)\,dt$.
The last display bounds the difference between the distribution
functions of $\widehat g_h(\cdot\mid x)$ and $g_h(\cdot\mid x)$.
If $D_j$ is that difference at the $j$th bin boundary, then
$D_0=D_B=0$ and $|D_j|\le\Delta_n(x)$. Consequently,
\[
 \|H_B\widehat g_h-H_Bg_h\|_1
 =\sum_{j=1}^B|D_j-D_{j-1}|
 \le2(B-1)\Delta_n(x),
\]
where $H_B$ is the bin-averaging operator used in
Appendix~\ref{app:kde-proof}. Its contraction property, together
with the smoothing and binning bounds established there, yields
\[
 \|\widetilde f_{h,B}(\cdot\mid x)-f(\cdot\mid x)\|_1
 \le2(B-1)\Delta_n(x)
    +\sqrt{\frac2\pi}(LK+2M_f)h+\frac{KL^2}{3B}.
\]

{
\paragraph{Quantile histograms.}
For this part, assume also $f\ge c>0$.
Continuity of $F_x$ implies that every jump of
}
$\widehat F_n(\cdot\mid x)$ is at most $2\Delta_n(x)$.
For linear interpolation of weighted CDF knots, the discrepancy
between the empirical CDF at the interpolated quantile and its
probability level is at most the largest jump. This remains true
with repeated response values or zero weights: a nonconstant
interpolation segment spans at most one positive CDF increment,
and equal response values contribute at most their combined atom.
Thus, for $0<\tau<1$,
\[
 |F_x(\widehat q_n^{\rm lin}(\tau\mid x))-\tau|
 \le3\Delta_n(x).
\]
Moving this quantile by at most $L\eta$ changes $F_x$ by at most
$M_fL\eta$. Hence, for the endpoints in
\eqref{eq:forest-quantile-endpoints},
\[
 \epsilon(x):=\max_{0\le j\le B}|F_x(z_j(x))-j/B|
 \le3\Delta_n(x)+M_fL\eta.
\]
Write $\ell_j=z_j(x)-z_{j-1}(x)$ and
$p_j=F_x(z_j(x))-F_x(z_{j-1}(x))$.
Comparison with the true average density in each bin gives
\[
 \|\widehat f^{\rm q}_{n,B,\eta}(\cdot\mid x)
             -f(\cdot\mid x)\|_1
 \le\sum_{j=1}^B|1/B-p_j|+\frac K3\sum_{j=1}^B\ell_j^2.
\]
The endpoint errors at $j=0,B$ are zero, so
\[
 \sum_{j=1}^B|1/B-p_j|\le2(B-1)\epsilon(x),
 \qquad
 c\ell_j\le p_j\le1/B+2\epsilon(x).
\]
Since $\sum_j\ell_j=L$,
\[
 \|\widehat f^{\rm q}_{n,B,\eta}(\cdot\mid x)
             -f(\cdot\mid x)\|_1
 \le2(B-1)\epsilon(x)
     +\frac{KL}{3c}\left(\frac1B+2\epsilon(x)\right).
\]
{
Taking expectations over training data and $Q_X$ proves
\eqref{eq:forest-density-bounds} and
\eqref{eq:forest-quantile-density-bound}.
}

{
\paragraph{Prediction-set length.}
{
Use \eqref{eq:uniform-mixture-invariance} and apply
Proposition~\ref{prop:fixed-m-density}, with $V=L$, to the normalized
unmixed density, using \eqref{eq:forest-density-bounds}.
}
This gives \eqref{eq:forest-kernel-length-bound} without assuming
$f\ge c$. Under the additional lower bound $f\ge c>0$,
\eqref{wo:density-positive-lower} applies to both the kernel and
the quantile density estimates. Together with their density bounds,
it gives \eqref{eq:forest-length-bounds}, uniformly over $m$.
}
\end{proof}
}

\paragraph{Convergence of $d_n$.}
Fix $P,Q$ with $Q_X\ll P_X$.
For the original-sample forest weights in
\eqref{eq:forest-training-weights},
Theorem~4.3 of \citet{elieditcosaque2022forest} gives
$\Delta_n(x)\to0$ almost surely for each fixed $x$
under its stated sampling conditions and
Assumptions~4.2--4.4.

Since $0\le\Delta_n(x)\le1$, bounded convergence gives
\[
 \mathbb E_{\rm train}\Delta_n(x)\longrightarrow0
 \qquad\text{for }P_X\text{-almost every }x.
\]
Because $Q_X\ll P_X$, this convergence also holds
for $Q_X$-almost every $x$.
Furthermore,
$0\le\mathbb E_{\rm train}\Delta_n(x)\le1$.
Interchanging the two expectations and applying
dominated convergence therefore yields
\[
 d_n
 =\int_{\mathcal X}
   \mathbb E_{\rm train}\Delta_n(x)\,Q_X(dx)
 \longrightarrow0.
\]

{
\paragraph{Empirical CDF error within groups.}
In Section~\ref{sec:vanishing-density}, each group $g\in\{0,1\}$
contains $n$ iid training responses $Y_{g,1},\ldots,Y_{g,n}$
with conditional CDF $F_g$. Define
\[
\widehat F_n(y\mid g)=\frac1n\sum_{i=1}^n
  \mathbf1\{Y_{g,i}\le y\},
\qquad
\Delta_n(g)=\sup_y|\widehat F_n(y\mid g)-F_g(y)|.
\]
The Dvoretzky--Kiefer--Wolfowitz--Massart inequality
\citep{massart1990dkw} gives, for every $t>0$,
\[
\mathbb P_{\rm train}\{\Delta_n(g)>t\}\le 2e^{-2nt^2}.
\]
Integrating this tail bound yields
\[
\mathbb E_{\rm train}\Delta_n(g)
=\int_0^\infty
  \mathbb P_{\rm train}\{\Delta_n(g)>t\}\,dt
\le\int_0^\infty 2e^{-2nt^2}\,dt
=\sqrt{\frac{\pi}{2n}}.
\]
Since $Q(X=0)=1/4$ and $Q(X=1)=3/4$, the error in
\eqref{eq:forest-cdf-error} satisfies
\begin{equation}\label{eq:group-cdf-rate}
d_n=\frac14\mathbb E_{\rm train}\Delta_n(0)
    +\frac34\mathbb E_{\rm train}\Delta_n(1)
\le\sqrt{\frac{\pi}{2n}}
=O(n^{-1/2}).
\end{equation}
No positive lower bound on the true densities is required.
This rate concerns the within-group empirical CDFs.
The kernel-histogram proof above also applies to these weights:
it uses only nonnegative, normalized training weights through
$\Delta_n(g)$.
}

{
\section{Selection from independent calibration samples}
\label{sec:independent-selection}
}

{
Use the scores and response domain from Section~\ref{sec:selection}.
Draw independent pairs $(T_\ell,X_\ell)\sim P^m\times Q_X$,
$\ell=1,\ldots,r$, independently of training; each $T_\ell$ contains
$m$ source observations. With $r_s$ from \eqref{eq:inclusion}, select
\begin{equation}\label{eq:selection}
\widehat R_k
=
\frac1r\sum_{\ell=1}^r
\int_a^b r_{s_k}(T_\ell,X_\ell;y)\,dy,
\qquad
\widehat k\in\operatorname*{arg\,min}_{1\le k\le M}
\widehat R_k,
\end{equation}
using a fixed tie rule. Return $C_{s_{\widehat k}}(T,X,U)$ with
fresh $T\sim P^m$ and $(X,Y)\sim Q$, independent of one another
and of training and selection.
}
\begin{theorem}[Selection and an oracle bound]
\label{thm:selection}
Conditional on the training data, for any $\delta\in(0,1)$,
with probability at least $1-\delta$ over the selection data,
\begin{equation}\label{eq:selection-bound}
R_m(C_{s_{\widehat k}})
\le
\min_{1\le k\le M}R_m(C_{s_k})
+
2L\sqrt{\frac{\log(2M/\delta)}{2r}}.
\end{equation}
Conditional on the training and selection data, the final
rule has coverage exactly $p$ under $P^m\times Q$.
\end{theorem}

{
\begin{proof}
Conditional on training, the summands in \eqref{eq:selection} are
independent, lie in $[0,L]$, and have mean $R_m(C_{s_k})$.
Hoeffding's inequality and a union bound give, with probability
at least $1-\delta$,
\[
 \max_k|\widehat R_k-R_m(C_{s_k})|
 \le\gamma,
 \qquad
 \gamma=L\sqrt{\frac{\log(2M/\delta)}{2r}}.
\]
The minimizing property of $\widehat k$ then gives
$R_m(C_{s_{\widehat k}})-\min_kR_m(C_{s_k})\le2\gamma$,
proving \eqref{eq:selection-bound}. Independent final calibration
gives exact coverage by Proposition~\ref{wo:exact}.
\end{proof}
}

{
\section{Proof of the pooled-selection bound}\label{app:pool-proof}

\begin{proof}[Proof of Theorem~\ref{thm:pool-selection}] 
{
Condition on the training data throughout. All scores and weights
are then fixed. For candidate $k$, write
\[
 \ell_k(T,x)=\int_a^b r_{s_k}(T,x;y)\,dy,
 \qquad h_k(T)=\mathbb E_{Q_X}\ell_k(T,X),
 \qquad R_k=R_m(C_{s_k}).
\]
Both lengths lie in $[0,L]$. The function $h_k$ is symmetric in the
$m$ calibration observations, and $\mathbb E_{P^m}h_k(T)=R_k$.
The minimizing property of $\widehat k_{\rm pool}$ implies
\begin{equation}\label{eq:pool-selection-reduction}
 R_m(C_{s_{\widehat k_{\rm pool}}})-\min_kR_k
 \le2\max_k|\widehat R_k^{\rm pool}-R_k|.
\end{equation}

\paragraph{A direct finite-$m$ bound.}
Assume first that $n_s\ge m$. We use the following elementary
consequence of an exponential-moment bound: if centered random
variables $\xi_1,\ldots,\xi_M$ satisfy
$\mathbb E\exp(t\xi_k)\le\exp(t^2\sigma^2/2)$ for every real $t$,
then
\[
 \mathbb E\max_{k\le M}|\xi_k|
 \le\sigma\sqrt{2\log(2M)}.
\]
Indeed, for $t>0$, the exponential of the maximum is bounded by
$\sum_k\{\exp(t\xi_k)+\exp(-t\xi_k)\}$; Jensen's inequality and
minimization over $t$ give the stated bound. Independence between
candidates is not needed.

Let $\widehat P_{n_s}=n_s^{-1}\sum_{i=1}^{n_s}\delta_{Z_i^{\rm sel}}$
be the empirical law of the labeled pool, and define
\[
 V_k=\mathbb E_{T\sim\widehat P_{n_s}^{\,m}}h_k(T),
 \qquad
 A_k=\frac1r\sum_{\ell=1}^rh_k(T_\ell^*).
\]
Conditional on the labeled pool, the summands defining $A_k$ are
independent, bounded by $L$, and have mean $V_k$. Hoeffding's
exponential-moment bound and the preceding maximum inequality give
\begin{equation}\label{eq:pool-bag-average-bound}
 \mathbb E\max_k|A_k-V_k|
 \le L\sqrt{\frac{\log(2M)}{2r}}.
\end{equation}

Average $h_k$ over all subsets of $m$ distinct pool indices:
\[
 U_k=\binom{n_s}{m}^{-1}
 \sum_{1\le i_1<\cdots<i_m\le n_s}
 h_k(Z_{i_1}^{\rm sel},\ldots,Z_{i_m}^{\rm sel}).
\]
The expectation of $U_k$ is $R_k$. Put $b_s=\lfloor n_s/m\rfloor$.
For each permutation $\pi$ of the $n_s$ indices, average $h_k$ over
the first $b_s$ disjoint blocks of $m$ observations in that order.
Each such block average has exponential moment, after centering,
at most $\exp\{t^2L^2/(8b_s)\}$. Its average over all permutations
is $U_k$, so convexity of the exponential gives the same bound
for $U_k-R_k$. Consequently,
\begin{equation}\label{eq:pool-u-statistic-bound}
 \mathbb E\max_k|U_k-R_k|
 \le L\sqrt{\frac{\log(2M)}{2\lfloor n_s/m\rfloor}}.
\end{equation}
In the product-law average defining $V_k$, all indices are distinct
with probability $1-\varepsilon_{\rm rep}(n_s,m)$. Conditional on
this event their average is $U_k$, by symmetry of $h_k$.
All other terms remain in $[0,L]$. Thus, for every realized pool,
\begin{equation}\label{eq:pool-collision-bound}
 \max_k|V_k-U_k|\le L\varepsilon_{\rm rep}(n_s,m).
\end{equation}
The upper bound in \eqref{eq:pool-repeat-probability} follows by a
union bound over the $\binom m2$ pairs of sampled indices.

For the input pool, set
\[
 \overline\ell_k(x)=\frac1r\sum_{\ell=1}^r\ell_k(T_\ell^*,x),
 \qquad A_k=\mathbb E_{Q_X}\overline\ell_k(X).
\]
This expression for $A_k$ agrees with its previous definition.
Conditional on the labeled pool and all bootstrap samples,
$\overline\ell_k$ is fixed and the input pool remains iid from $P_X$.
Since $dQ_X=w\,dP_X$ and $\mathbb E_Pw=1$,
\[
 \widehat R_k^{\rm pool}-A_k
 =\frac{q^{-1}\sum_{j=1}^q
       w(X'_j)\{\overline\ell_k(X'_j)-A_k\}}
       {q^{-1}\sum_{j=1}^qw(X'_j)}.
\]
Each numerator summand has mean zero and belongs to the interval
$[-w_+A_k,\,w_+(L-A_k)]$, of length $w_+L$.
The denominator is at least $w_-$. The same exponential-moment
argument therefore yields
\begin{equation}\label{eq:pool-input-max-bound}
 \mathbb E\max_k|\widehat R_k^{\rm pool}-A_k|
 \le\kappa L\sqrt{\frac{\log(2M)}{2q}}.
\end{equation}
The triangle inequality, \eqref{eq:pool-bag-average-bound}--
\eqref{eq:pool-input-max-bound}, and
\eqref{eq:pool-selection-reduction} prove
\eqref{eq:pool-selection-bound}.

\paragraph{A complementary bound when $f\ge c>0$.}
The following argument allows every $n_s,m,q,r\ge1$.

}\paragraph{Weighted empirical distribution functions.}
First consider $n$ iid observations from any distribution $R$ on
$\mathcal X\times\mathcal Y$, with the weights in
\eqref{eq:pool-weight-bound}. For a fixed score $s$, put
\[
H_R(t)=\frac{\mathbb E_R[w(X)\mathbf1\{s(X,Y)\le t\}]}
                 {\mathbb E_Rw(X)},
\qquad
\widehat H_n(t)=
\frac{\sum_{i=1}^n w(X_i)\mathbf1\{s(X_i,Y_i)\le t\}}
     {\sum_{i=1}^n w(X_i)}.
\]
Let $D_n$ be the supremum error of the unnormalized numerator,
including its limit at $t=+\infty$. Symmetrization gives
\[
\mathbb E D_n
\le\frac2n\mathbb E\sup_t
 \left|\sum_{i=1}^n\epsilon_i w(X_i)
                    \mathbf1\{s(X_i,Y_i)\le t\}\right|,
\]
where the $\epsilon_i$ are independent random signs. Conditional on
the observations, sort them by score. The displayed sums form a subset
of the partial sums of independent, mean-zero random variables.
The $L^2$ maximal inequality for these partial sums bounds the
conditional expected maximum by
$2\{\sum_i w(X_i)^2\}^{1/2}\le2w_+\sqrt n$.
Thus $\mathbb E D_n\le4w_+/\sqrt n$.
Since the empirical denominator is at least $w_-$ after division by $n$,
and the error in the mean weight is at most $D_n$,
\begin{equation}\label{eq:pool-cdf-error}
\mathbb E\|\widehat H_n-H_R\|_\infty
\le\frac{2\mathbb E D_n}{w_-}
\le\frac{8\kappa}{\sqrt n}.
\end{equation}
This argument permits arbitrary ties and also applies conditionally
when $R$ is the empirical distribution of the labeled pool.

\paragraph{Length stability under calibration.}
Fix a score $s$ and write
$H(t)=Q\{s(X,Y)\le t\}=H_P(t)$.
For any calibration bag $T$ of size $m$, let $H_T$ be its normalized
weighted score distribution function and set
$\delta_T=\|H_T-H\|_\infty$.
Define
\[
\ell_s(T)=\mathbb E_{Q_X,U}|C_s(T,X,U)|.
\]
For comparison, use the population rule with inclusion event
\[
H(s(x,y)-)+V\{H(s(x,y))-H(s(x,y)-)\}<p,
\qquad V\sim\operatorname{Unif}[0,1],
\]
and denote its target expected length by $\ell_s^\circ$.
Here $H(t-)$ denotes the left limit at $t$.

Use $1-V$ as the auxiliary randomization in
\eqref{eq:calibrated}. Its inclusion event can then be written as
\[
H_T(S-)+V\{H_T(S)-H_T(S-)\}
 +(V-p)\frac{w(X)}{\sum_{i=1}^m w(X_i)}<p,
\qquad S=s(X,Y).
\]
Its left side differs from
$H(S-)+V\{H(S)-H(S-)\}$ by at most
$\delta_T+\kappa/m$. Under $Q$ and independent $V$, the latter
random variable is uniform on $[0,1]$, including when $S$ has atoms.
The two inclusion events therefore disagree with probability at most
$2(\delta_T+\kappa/m)$. Since $f\ge c$, integration with respect to
$Q_X(dx)\,dy\,dV$ is at most $1/c$ times integration with respect to
$Q(dx,dy)\,dV$. Consequently,
\begin{equation}\label{eq:pool-length-stability}
|\ell_s(T)-\ell_s^\circ|
\le\frac2c\left(\delta_T+\frac\kappa m\right).
\end{equation}
For a fresh iid calibration sample, \eqref{eq:pool-cdf-error} implies
\begin{equation}\label{eq:pool-fresh-risk}
|R_m(C_s)-\ell_s^\circ|
\le\frac2c\left(\frac{8\kappa}{\sqrt m}
                         +\frac\kappa m\right).
\end{equation}

\paragraph{The resampled calibration bags.}
Let $\widehat H_{n_s}$ be the weighted score distribution function of
the source labeled pool. By \eqref{eq:pool-cdf-error}, its expected
distance from $H$ is at most $8\kappa/\sqrt{n_s}$.
Conditional on this pool, each bootstrap bag is iid from its empirical
distribution, so
\[
\mathbb E\|H_{T_\ell^*}-H\|_\infty
\le\frac{8\kappa}{\sqrt{n_s}}+\frac{8\kappa}{\sqrt m}.
\]
{
Use $\overline\ell_k$ and $A_k$ as defined above.
Jensen's inequality and \eqref{eq:pool-length-stability} give
}
\begin{equation}\label{eq:pool-bootstrap-risk}
\mathbb E|A_k-\ell_{s_k}^\circ|
\le\frac2c\left(
 \frac{8\kappa}{\sqrt{n_s}}+\frac{8\kappa}{\sqrt m}
 +\frac\kappa m\right).
\end{equation}
This bound holds for every $r\ge1$; increasing $r$ is not required.

\paragraph{The input pool and selection.}
Conditional on the labeled pool and bootstrap bags, the centered-ratio
identity used for \eqref{eq:pool-input-max-bound} has denominator at
least $w_-$. Its numerator is the average of $q$ independent,
mean-zero variables in an interval of length $w_+L$. Cauchy--Schwarz
and the variance bound for a bounded variable therefore give
\[
 \mathbb E|\widehat R_k^{\rm pool}-A_k|
 \le\frac{L\kappa}{2\sqrt q}
 \le\frac{L\kappa}{\sqrt q}.
\]
Together with \eqref{eq:pool-fresh-risk} and
\eqref{eq:pool-bootstrap-risk}, this proves
\[
\mathbb E_{\rm sel}
\left|\widehat R_k^{\rm pool}-R_m(C_{s_k})\right|
\le\beta_{n_s,q,m}.
\]
{
Bound the maximum in \eqref{eq:pool-selection-reduction} by the
sum and take expectations to obtain \eqref{eq:pool-growing-m-bound}.
For either selection bound, add
$\min_kR_m(C_{s_k})-R_m(C^*)$ and average over training data.
This gives \eqref{eq:pool-end-to-end} and its stated alternative.
Finally, conditional on training and selection, the chosen score is
fixed and the final calibration and test observations are independent
of those data. Proposition~\ref{wo:exact} gives coverage exactly $p$.
}
\end{proof}
}

{
{
\section{Proofs of oracle-length consistency at fixed calibration size}
\label{app:forest-fixed-m-proof}
}

{
\begin{proof}[Proof of Corollary~\ref{cor:forest-fixed-m-oracle}] 
{
For fixed $m$, Proposition~\ref{prop:forest-density} gives the bound
in \eqref{eq:forest-fixed-m-oracle}. All three terms of
$\mathcal A^{\rm ker}_{n,h_n,B_n}$ vanish by
\eqref{eq:forest-fixed-m-growth}, since
$M_f\le L^{-1}+KL/2$ is independent of $n$.
The score is finite and measurable conditional on training, so
independent final calibration and test data give exact coverage
by Proposition~\ref{wo:exact}.
}
\end{proof}
}

{
\subsection{Consistency after selection at fixed calibration size}
\label{app:forest-selected-fixed-m-proof}

\begin{proof}[Proof of Corollary~\ref{cor:forest-selected-fixed-m}] 
For each included kernel candidate $j\le J_n$,
Proposition~\ref{prop:forest-density} bounds its expected density
error by
\[
 2(2^j-1)d_n+D2^{-j},
 \qquad
 D=\sqrt{\frac2\pi}(LK+2M_f)L+\frac{KL^2}{3},
\]
where $M_f\le L^{-1}+KL/2$. Proposition~\ref{prop:fixed-m-density}
and \eqref{eq:pool-end-to-end} therefore give, whenever $n_s\ge m$,
\begin{equation}\label{eq:forest-selected-fixed-m-bound}
\begin{aligned}
 &\mathbb E_{\rm train,sel}
 [R_m(C_{s_{\widehat k_{\rm pool}}})-R_m(C^*)]\\
 &\qquad\le mL\min_{1\le j\le J_n}
       \{2(2^j-1)d_n+D2^{-j}\}
       +\gamma_{n_s,q,m,r_n,M_n}.
\end{aligned}
\end{equation}
{
By \eqref{eq:uniform-mixture-invariance}, density error is evaluated
before mixing. Additional scores can only improve the best-candidate
comparator.
}

Fix $j$. Since $J_n\to\infty$, this candidate is eventually
included, and $d_n\to0$ makes its limiting approximation term
at most $mLD2^{-j}$. Let $j\to\infty$.
Condition \eqref{eq:forest-selected-fixed-m-growth} also implies
$n_s,q,r_n\to\infty$. Since $m$ is fixed, eventually $n_s\ge m$,
and every term of \eqref{eq:pool-finite-m-budget} tends to zero.
Equation~\eqref{eq:forest-selected-fixed-m-bound} proves
\eqref{eq:forest-selected-fixed-m}.
Exact coverage follows from Theorem~\ref{thm:pool-selection} and
independence of the final calibration sample.
\end{proof}
}
}

{
\section{Oracle length after selection with growing calibration size}
\label{sec:forest-selection-consistency}
\label{app:forest-selection-proof}

{
\begin{remark}[Selection with growing calibration size]
\label{cor:forest-selected-oracle}
Assume \eqref{eq:regularity} and \eqref{eq:pool-weight-bound} with
constants independent of sample size, and suppose the forest weights
satisfy $d_n\to0$. Include at least one of the following families
constructed from these weights, for $j=1,\ldots,J_n$:
\begin{equation}\label{eq:forest-expanding-library}
\begin{aligned}
 \text{kernel histograms:}\quad
 &B_j=2^j,\quad h_j=L2^{-j},\quad \lambda\in[0,1),\\
 \text{quantile histograms:}\quad
 &B_j=2^j,\quad \eta_j=2^{-2j}.
\end{aligned}
\end{equation}
Within each kernel candidate, $\lambda$ is common to all $(x,y)$.
Additional training-fitted scores are allowed; let $M_n$ be the total
library size. Fit all scores independently of selection, final
calibration, and test data. Use \eqref{eq:pool-selection} with
sizes $n_s,q,m_n$.
If
\begin{equation}\label{eq:forest-selection-growth}
 J_n\to\infty,
 \qquad
 M_n(n_s^{-1/2}+q^{-1/2}+m_n^{-1/2})\to0,
\end{equation}
then, for any fixed $r\ge1$, including $r=30$,
\begin{equation}\label{eq:forest-selected-oracle}
 \mathbb E_{\rm train,sel}
 [R_{m_n}(C_{s_{\widehat k_{\rm pool}}})-R_{m_n}(C^*)]\to0.
\end{equation}
The oracle uses the current size $m_n$. Independent final calibration
gives exact target marginal coverage $p$ at each sample size.
\end{remark}
\begin{proof}
For a level-$j$ candidate, Proposition~\ref{prop:forest-density} gives
$\mathbb E_{\rm train}[R_{m_n}(C_{s_{n,j}})-R_{m_n}(C^*)]
\le C(2^jd_n+2^{-j})$, with $C$ independent of $n,j,m_n$.
For quantile histograms this uses $2^j\eta_j=2^{-j}$.
Fix $j$, let $n\to\infty$, and then let $j\to\infty$; the
best-candidate gap tends to zero. Condition
\eqref{eq:forest-selection-growth} also gives
$M_n\beta_{n_s,q,m_n}\to0$.
The complementary bound in Theorem~\ref{thm:pool-selection} proves
\eqref{eq:forest-selected-oracle} and the coverage statement.
\end{proof}
}

{
\section{A linear density-error bound at fixed calibration size}
\label{app:fixed-m-density-proof}

\subsection{Proof of Proposition~\ref{prop:fixed-m-density}}

Write $V=\mu(\mathcal Y)$ and assume $0<V<\infty$.
All densities below are jointly measurable, finite, nonnegative, and
integrate to one for $P_X$-almost every input. On exceptional input
fibers, choose the uniform density $1/V$.
For a conditional density $h$, let
\[
 P_h(dx,dy)=P_X(dx)h(y\mid x)\,\mu(dy),
 \qquad
 R_{m,h}(C_s)
 =\mathbb E_{T\sim P_h^m,\,X\sim Q_X}
       \int_{\mathcal Y}r_s(T,X;y)\,\mu(dy).
\]
The input laws and weight function remain fixed as $h$ varies, and
$R_{m,f}=R_m$. The loss in this display integrates out the conformal
randomization and lies in $[0,V]$. It does not involve a test response.

We first prove a sharper weighted bound. Define
\begin{equation}\label{eq:fixed-m-weighted-error}
\begin{aligned}
 g(x)&=\|\widetilde f(\cdot\mid x)-f(\cdot\mid x)\|_{L^1(\mu)},\\
 e_P&=\mathbb E_{P_X}g(X),\qquad e_Q=\mathbb E_{Q_X}g(X)=e,\\
 \phi_w(u)&=\mathbb E_{X'\sim P_X}\min\{w(X'),u\},\qquad u\ge0,\\
 J_w(\widetilde f,f)
 &=\mathbb E_{P_X}[\phi_w(w(X))g(X)].
\end{aligned}
\end{equation}
Since $\mathbb E_{P_X}w(X)=1$,
\begin{equation}\label{eq:fixed-m-weighted-comparison}
 0\le\phi_w(u)\le\min\{1,u\},
 \qquad J_w(\widetilde f,f)\le\min\{e_P,e_Q\}.
\end{equation}
We will show
\begin{equation}\label{eq:fixed-m-weighted-bound}
 0\le R_m(C_{-\widetilde f})-R_m(C^*)
 \le\min\{V,mVJ_w(\widetilde f,f)\}.
\end{equation}
No positive lower or finite uniform upper bound on $w$ is needed.

\paragraph{Changing one calibration response.}
Fix a score $s$, all calibration inputs, and a query $(x,y)$ with
$v=w(x)>0$. Replace only calibration response $Y_i$, whose input has
weight $u=w(X_i)$. Put
\[
\begin{aligned}
 D&=\sum_{j\ne i}w(X_j)\mathbf1\{s(X_j,Y_j)<s(x,y)\},\\
 B&=v+\sum_{j\ne i}w(X_j)\mathbf1\{s(X_j,Y_j)=s(x,y)\},
 \qquad z=pW_x-D.
\end{aligned}
\]
Here $B\ge v$ and $z$ is unchanged by the replacement. According as
the $i$th score is below, equal to, or above the query score,
\eqref{eq:inclusion} equals
\[
 r_< =\clip{\frac{z-u}{B}},\qquad
 r_= =\clip{\frac{z}{B+u}},\qquad
 r_> =\clip{\frac{z}{B}}.
\]
These values satisfy $r_<\le r_=\le r_>$.
For $z\le0$ they are all zero, and for $z\ge B+u$ they are all one.
In the remaining range, the left inequality is immediate if $z\le u$;
otherwise it follows from $(z-u)/B\le z/(B+u)$, which is equivalent
to $z\le B+u$. The right inequality follows by comparing denominators
and clipping. As clipping is $1$-Lipschitz,
$r_>-r_<\le\min\{1,u/B\}\le\min\{1,u/v\}$.
Writing $\ell_s(T,x)=\int r_s(T,x;y)\,\mu(dy)$, we obtain
\begin{equation}\label{eq:fixed-m-response-sensitivity}
 |\ell_s(T,x)-\ell_s(T',x)|
 \le V\min\{1,w(X_i)/w(x)\}\mathbf1\{Y_i\ne Y_i'\}.
\end{equation}
This calculation includes score ties and zero-weight calibration points.

\paragraph{Changing the calibration law.}
Let $h$ be another normalized conditional density. For each input $x$,
put $a_x(y)=\min\{f(y\mid x),h(y\mid x)\}$ and
$d(x)=\tfrac12\|h(\cdot\mid x)-f(\cdot\mid x)\|_{L^1(\mu)}$.
A conditional coupling of the two response laws is
\[
 K_x(dy,dy')=a_x(y)\mu(dy)\delta_y(dy')
 +\frac{[f(y\mid x)-a_x(y)][h(y'\mid x)-a_x(y')]}{d(x)}
       \mu(dy)\mu(dy'),
\]
where the second term is zero when $d(x)=0$.
The first term has mass $1-d(x)$ and the second has mass $d(x)$;
their marginals are the required laws. The residual densities have
disjoint supports, so
$\mathbb P(Y\ne Y'\mid X=x)=d(x)$.
Joint measurability and the standard Borel assumption make $K_x$ a
measurable probability kernel.

Draw shared inputs $X_1,\ldots,X_m$ iid from $P_X$ and, conditionally
independently, draw the response pairs from $K_{X_i}$.
Replace their responses one at a time and apply
\eqref{eq:fixed-m-response-sensitivity}. For $u\ge0$,
\[
 \mathbb E_{X\sim Q_X}\min\{1,u/w(X)\}
 =\mathbb E_{X\sim P_X}\min\{w(X),u\}=\phi_w(u);
\]
the set $w(X)=0$ contributes zero. The triangle inequality and Tonelli's
theorem therefore give, for every fixed score $s$,
\begin{equation}\label{eq:fixed-m-law-perturbation}
 |R_{m,f}(C_s)-R_{m,h}(C_s)|
 \le\frac{mV}{2}\,
 \mathbb E_{P_X}\!\left[
  \phi_w(w(X))\|h(\cdot\mid X)-f(\cdot\mid X)\|_{L^1(\mu)}
 \right].
\end{equation}
This bound is uniform over the fixed score. Intermediate samples in
the replacement argument need not be iid, since
\eqref{eq:fixed-m-response-sensitivity} is deterministic.

\paragraph{Comparing the two oracle rules.}
Suppose first that $\widetilde f>0$ almost everywhere.
Under $P_{\widetilde f}$ and its shifted target law, the validity class
is still $\mathcal H_w$: its definition depends on $w$ and $m$, not
on the conditional density. Theorem~\ref{wo:oracle} thus gives
\[
 R_{m,\widetilde f}(C_{-\widetilde f})
 \le R_{m,\widetilde f}(C^*).
\]
Adding and subtracting these two risks and using
\eqref{eq:fixed-m-law-perturbation} for each score yields
\[
\begin{aligned}
 R_m(C_{-\widetilde f})-R_m(C^*)
 &\le |R_{m,f}(C_{-\widetilde f})
           -R_{m,\widetilde f}(C_{-\widetilde f})|\\
 &\quad+|R_{m,\widetilde f}(C^*)-R_{m,f}(C^*)|\\
 &\le mVJ_w(\widetilde f,f).
\end{aligned}
\]
True-law oracle optimality gives nonnegativity, and the bounded loss
gives the upper bound $V$.

If $\widetilde f$ can vanish, replace it only for this argument by
$\widetilde f_\epsilon=(1-\epsilon)\widetilde f+\epsilon/V$,
where $0<\epsilon<1$. This common affine transformation preserves
every density comparison and tie, including comparisons across inputs.
Consequently $C_{-\widetilde f_\epsilon}=C_{-\widetilde f}$ for
every sample and every conformal randomization.
Moreover,
$\|\widetilde f_\epsilon(\cdot\mid x)
      -\widetilde f(\cdot\mid x)\|_{L^1(\mu)}\le2\epsilon$,
so $J_w(\widetilde f_\epsilon,f)\to J_w(\widetilde f,f)$.
Apply the positive-density result and let $\epsilon\downarrow0$.
This proves \eqref{eq:fixed-m-weighted-bound} without changing the
delivered prediction rule.

By \eqref{eq:fixed-m-weighted-comparison}, the new term is at most
$mVe$. Combining it with Corollary~\ref{wo:density-ordering} proves
\eqref{eq:fixed-m-density-bound}. Conditional on independent training,
the same argument applies to each fitted density. Averaging each
upper bound separately, and then taking the infimum over $\tau$,
proves \eqref{eq:fixed-m-density-mean}. In particular, fixed $m$ and
$\bar e_n=O(a_n)$ imply expected excess length $O(a_n)$.
Exact conditional coverage follows from Proposition~\ref{wo:exact}.

\subsection{Sharpness of the linear density-error order}
\label{app:density-linear-sharpness}

Take a single input value, $w=1$, $\mathcal Y=[0,1]$, Lebesgue measure,
$m=1$, and $p=1/2$. For $0<\delta\le1/2$, set
\begin{equation}\label{eq:fixed-m-linear-witness}
 f_\delta(y)=1+\delta(2y-1),\qquad
 \widetilde f_\delta(y)=1-\delta(2y-1),\qquad 0\le y\le1.
\end{equation}
Both densities are smooth, take values in $[1/2,3/2]$, and have
derivatives bounded in magnitude by one. Their error is
\[
 e=\int_0^1|\widetilde f_\delta(y)-f_\delta(y)|\,dy
   =2\delta\int_0^1|2y-1|\,dy=\delta.
\]
For $m=1$ and $p=1/2$, a query score below the calibration score is
included with probability one, a score above it with probability zero,
and a tie with probability $1/2$. The oracle density is increasing,
so its set has length $1-Y_1$; the fitted density is decreasing,
so its set has length $Y_1$, almost surely in $U$. Ties occur only at $y=Y_1$ and have
zero Lebesgue measure. Since
\[
 \mathbb E_{f_\delta}Y_1
 =\int_0^1y[1+\delta(2y-1)]\,dy=\frac12+\frac\delta6,
\]
we obtain
\begin{equation}\label{eq:fixed-m-linear-witness-gap}
 R_m(C^*)=\frac12-\frac\delta6,\qquad
 R_m(C_{-\widetilde f_\delta})=\frac12+\frac\delta6,
 \qquad
 R_m(C_{-\widetilde f_\delta})-R_m(C^*)=\frac e3.
\end{equation}
Thus no uniform $o(e)$ bound holds over these smooth density pairs,
even with a common positive density lower bound.
The true density varies with $\delta$; this is not a local lower bound
at one fixed density. The example does not establish optimal dependence
on $m$, $p$, or $w$, an optimal constant, or a statistical minimax rate.

\subsection{Why the previous bound remains useful}

The two bounds control different regimes. With $w=1$ and
$\mathcal Y=[0,1]$, if $f(y)=1$, then
\[
 \inf_{\tau>0}\{e/\tau+\mathcal T_f(\tau)\}=\min\{e,1\}.
\]
For $m>1$ and $0<e<1/m$, this is smaller than $me$.
In contrast, for $f(y)=2y$,
\[
 \mathcal T_f(\tau)=
 \begin{cases}
   \tau/4,&0<\tau\le2,\\
   1-1/\tau,&\tau\ge2,
 \end{cases}
 \qquad
 \inf_{\tau>0}\{e/\tau+\mathcal T_f(\tau)\}=\sqrt e
 \quad(0\le e\le1).
\]
Hence $me<\sqrt e$ when $0<e<1/m^2$.
These comparisons concern the upper bounds, not attained risks.
The older bound is uniform in $m$, whereas the new linear term is
particularly useful when $m$ is fixed and density error is small.
}

\FloatBarrier
\bibliography{references}

}
}
}
\end{document}